\documentclass{article}

    \usepackage[preprint,nonatbib]{neurips_2023}

\usepackage[utf8]{inputenc} 
\usepackage[T1]{fontenc}    
\usepackage{hyperref}       
\usepackage{url}            
\usepackage{booktabs}       
\usepackage{amsfonts}       
\usepackage{nicefrac}       
\usepackage{microtype}      
\usepackage{xcolor}         

\usepackage{amsmath}
\usepackage{amsthm}
\usepackage{amssymb}

\usepackage{lipsum}
\usepackage{multirow}
\usepackage{stmaryrd}

\usepackage{multirow}
\usepackage{bbm}
\usepackage{graphicx}
\usepackage{amssymb}
\usepackage{adjustbox}
\usepackage{xspace}
\usepackage{algpseudocode}
\usepackage{algorithm}
\usepackage{graphicx}
\usepackage{caption}
\usepackage{subcaption}
\usepackage{color, colortbl}
\usepackage{tablefootnote}
\usepackage{enumitem}  
\usepackage{lipsum}
\usepackage{xspace}

\usepackage{makecell}
\usepackage{algpseudocode}
\usepackage{footnote}
\makesavenoteenv{tabular}
\usepackage{wrapfig}   

\usepackage{pifont}

\setlist{topsep=1pt,itemsep=1pt,partopsep=1pt, parsep=1pt}

\usepackage[most]{tcolorbox} 

\definecolor{bluex}{rgb}{0.27, 0.42, 0.81}
\definecolor{purplex}{HTML}{9564bf}
\definecolor{red3}{HTML}{C52A20}
\definecolor{red2}{HTML}{B36A6F}
\definecolor{red1}{HTML}{FFb5b5}
\definecolor{purple}{HTML}{B36A6F}
\definecolor{darkyellow}{HTML}{D5BA82}
\definecolor{blue1}{HTML}{508AB2}
\definecolor{blue2}{HTML}{C4E4E3}
\definecolor{green1}{HTML}{A1D0C7}
\definecolor{green2}{HTML}{BFF6BA}
\definecolor{green3}{HTML}{028100}
\definecolor{teal}{HTML}{508AB2}
\definecolor{purple1}{HTML}{8d3a94}

\newtcbtheorem[]{gpt4answer}{GPT-4 Answer}%
{colback=blue2!5,colframe=blue1!80,fonttitle=\bfseries, left=.02in, right=.02in,bottom=.02in, top=.02in,breakable}{gpt4answer}

\newtcbtheorem[]{userinput}{User Input}%
{colback=green2!5,colframe=green3!80,fonttitle=\bfseries, left=.02in, right=.02in,bottom=.02in, top=.02in,breakable}{userinput}

\title{Enhancing Q-Learning with \\ Large Language Model Heuristics}
\author{%
  Xiefeng Wu \\
  School of Computer Science \\
  Wuhan University \\
  \texttt{wuxiefeng@whu.edu.cn} \\
}

\begin{document}

\newcommand{\yes}{\checkmark}
\newcommand{\hval}{\mathbf{h}}
\newcommand{\mdp}{\mathcal{M}}
\newcommand{\sspace}{\mathcal{S}}
\newcommand{\aspace}{\mathcal{A}}
\newcommand{\saspace}{\mathcal{Z}}
\newcommand{\dataspace}{\mathcal{D}}\newcommand{\pispace}{\Pi}
\newcommand{\rawrew}{\mathcal{R}}
\newcommand{\rawerew}{\mathbf{r}}
\newcommand{\rawtrns}{P}
\newcommand{\rew}{\rawrew}
\newcommand{\erew}{\rawerew}
\newcommand{\trns}{\rawtrns}
\newcommand{\rewm}{\rawrew_{\mdp}}
\newcommand{\erewm}{\rawerew_{\mdp}}
\newcommand{\trnsm}{\rawtrns_{\mdp}}
\newcommand{\val}{\mathbf{v}}
\newcommand{\qval}{\mathbf{q}}
\newcommand{\vald}{\val_D}
\newcommand{\lowval}{\val}
\newcommand{\lowqval}{\qval}
\newcommand{\highval}{\widehat{\val}}
\newcommand{\optplow}{\widecheck{\pi}^*}
\newcommand{\optplowd}{\optplow_D}
\newcommand{\disctrns}{\Lambda}
\newcommand{\disctrnsd}{\Lambda_D}
\newcommand{\uncraw}{\mathbf{u}}
\newcommand{\uncsa}{\uncraw_{D,\delta}}
\newcommand{\unc}{\uncraw_{D,\delta}^{\pi}}
\newcommand{\uncp}{\uncraw_{D,\delta}^{\pi'}}

\newcommand{\countd}{\dot{\mathbf{n}}_D}
\newcommand{\ircountd}{\countd^{-\frac12}}
\newcommand{\countdsa}{\ddot{\mathbf{n}}_D}
\newcommand{\ircountdsa}{\countdsa^{-\frac12}}

\newcommand{\lowvalopt}{\lowval^*}
\newcommand{\lowvaloptd}{\lowvalopt_D}
\newcommand{\highvalopt}{\highval^*}
\newcommand{\highvaloptd}{\highvalopt_D}

\newcommand{\Id}{I}
\newcommand{\rewpi}{r^{\pi}}
\newcommand{\trnspi}{\trns^{\pi}}
\newcommand{\act}{A}
\newcommand{\actpi}{\act^\pi}
\newcommand{\actmu}{\act^\mu}

\newcommand{\actpip}{\act^{\pi'}}
\newcommand{\valpi}{\val^{\pi}}
\newcommand{\qvalpi}{\qval^{\pi}}
\newcommand{\bell}{\mathcal{B}}
\newcommand{\bellpi}{\bell^{\pi}}
\newcommand{\trnspid}{\trnspi_D}
\newcommand{\bellpid}{\bellpi_D}
\newcommand{\bellpidup}{\overline{\bellpi_D}}
\newcommand{\bellpidlow}{\underline{\bellpi_D}}
\newcommand{\bellpiund}{\bellpi_{D,\unkval}}

\newcommand{\subopt}{\textsc{SubOpt}}

\newcommand{\valset}{\mathbb{V}}
\newcommand{\valsetpi}{\mathbb{V}^{\pi}}
\newcommand{\valpid}{\val^{\pi}_D}
\newcommand{\qvalpid}{\qval^{\pi}_D}
\newcommand{\valsetpid}{\mathbb{V}^{\pi}_D}
\newcommand{\rvalpid}{\bar\val^{\pi}_D}
\newcommand{\lowvalpi}{\lowval^\pi}
\newcommand{\lowqvalpi}{\lowqval^\pi}
\newcommand{\highvalpi}{\highval^\pi}
\newcommand{\lowvald}{\lowval_D}
\newcommand{\lowvalpid}{\lowval^\pi_D}
\newcommand{\lowqvalpid}{\lowqval^\pi_D}
\newcommand{\highvalpid}{\highvalpi_D}
\newcommand{\disctrnspi}{\disctrns^\pi}
\newcommand{\disctrnspid}{\disctrnsd^\pi}
\newcommand{\fatvalsetpid}{\mathbb{\widebar V}^{\pi}_D}
\newcommand{\fatvalsetoptd}{\mathbb{\widebar V}^{*}_D}

\newcommand{\optp}{\pi^{*}}
\newcommand{\valopt}{\val^{*}}
\newcommand{\bellopt}{T^{*}}
\newcommand{\trnsoptd}{\trnsopt_D}
\newcommand{\belloptdup}{\overline{\bellopt_D}}
\newcommand{\belloptdlow}{\underline{\bellopt_D}}
\newcommand{\belloptund}{\bellopt_{D,\unkval}}
\newcommand{\valoptd}{\val^{\optpd}_D}
\newcommand{\valsetoptd}{\mathbb{V}^*_D}
\newcommand{\optpd}{\pi_D^{*}}
\newcommand{\lowoptpd}{\underline{\pi}_D^{*}}

\newcommand{\rewd}{\rawrew_D}
\newcommand{\erewd}{\rawerew_D}
\newcommand{\trnsd}{\rawtrns_D}
\newcommand{\real}{\mathbb{R}}
\newcommand{\dist}{{Dist}}
\newcommand{\onehalf}{\frac{1}{2}}
\newcommand{\zerov}{\dot{\mathbf{0}}}
\newcommand{\zerovsa}{\ddot{\mathbf{0}}}
\newcommand{\onev}{\dot{\mathbf{1}}}
\newcommand{\onevsa}{\ddot{\mathbf{1}}}

\newcommand{\err}{\text{err}}
\newcommand{\rewerr}{\err_{\rew}}
\newcommand{\trnserr}{\err_{\trns}}

\newcommand{\E}{\mathbb{E}}

\newcommand{\emp}{\hat{\pi}_{D}}
\newcommand{\tv}{\text{TV}}

\newtheorem{theorem}{Theorem}
\newtheorem{lemma}{Lemma}
\newtheorem{claim}{Claim}
\newtheorem{assumption}{Assumption}
\newtheorem{corollary}{Corollary}
\newtheorem{definition}{Definition}
\newtheorem{proposition}{Proposition}
\newtheorem{remark}{Remark}

\newcommand{\sa}{\langle s,a \rangle}
\newcommand{\smu}{\langle s,\mu(s) \rangle}

\newcommand{\vunc}{\boldsymbol{\mu}}
\newcommand{\vuncpid}{\vunc_{D,\delta}^{\pi}}
\newcommand{\vuncpipd}{\vunc_{D,\delta}^{\pi'}}
\newcommand{\vuncempd}{\vunc_{D,\delta}^{\emp}}

\newcommand{\norm}[1]{\left\lVert#1\right\rVert}
\newcommand{\argmax}{\arg \max}
\newcommand{\argmin}{\arg \min}
\newcommand{\y}{\mathbf{y}}
\newcommand{\x}{\mathbf{x}}
\newcommand{\f}{\mathbf{f}}
\newcommand{\bvec}{\mathbf{b}}
\newcommand{\uvec}{\mathbf{u}}

\maketitle

\begin{abstract}

Q-learning excels in learning from feedback within sequential decision-making tasks but often requires extensive sampling to achieve significant improvements. While reward shaping can enhance learning efficiency, non-potential-based methods introduce biases that affect performance, and potential-based reward shaping, though unbiased, lacks the ability to provide heuristics for state-action pairs, limiting its effectiveness in complex environments. Large language models (LLMs) can achieve zero-shot learning for simpler tasks, but they suffer from low inference speeds and occasional hallucinations. To address these challenges, we propose \textbf{LLM-guided Q-learning}, a framework that leverages LLMs as heuristics to aid in learning the Q-function for reinforcement learning. Our theoretical analysis demonstrates that this approach adapts to hallucinations, improves sample efficiency, and avoids biasing final performance. Experimental results show that our algorithm is general, robust, and capable of preventing ineffective exploration.
\end{abstract}

\section{Introduction}
\label{sec:intro}

Q-learning with function approximation is an effective algorithm for addressing sequential decision-making problems, initially proposed by \cite{watkins1992q}. In many popular Actor-Critic algorithms (A2C) \cite{ddpg,sac,td3}, the Q-function is crucial for deriving policies, thereby implicitly guiding agent exploration. However, applying TD updates to learn the Q-function from the experience buffer presents several challenges, including potential divergence \cite{chen2023target_truncation} and high sampling demands \cite{sac,td3}.

Previous research has sought to enhance sample efficiency through reward shaping but failed in handling wrong heuristic.\cite{cheng2021heuristic} applies a decay factor in the reward heuristic \((1 - \lambda )\gamma \mathbb{E}_{s' \sim P}[h(s')])\) to reduce inaccuracies in guidance. \cite{Hu_Wang_Jia_Wang_Chen_Hao_Wu_Cheng_2019} applied \( z(s, a) \) to the reward heuristic \(z(s, a)f(s, a) \) to incorporate shaping rewards and adapt the shaping function \( f(s, a) \).
Both methods require additional learning steps to accommodate inaccuracies in the guidance.

Using LLMs/VLMs as agents represents a viable approach for planning and control. \cite{brooks2024large} employs an LLM as a world model, allowing it to directly output rewards and states. Similarly, \cite{wang2023voyager} utilizes an LLM to generate skill-based Python code, enabling autonomous exploration and skill expansion. Studies like \cite{shi2024yell}, \cite{liu2023interactive}, and \cite{ouyang2024long} use LLMs as high-level controllers to enhance robots' interaction abilities and scene understanding. However, all these methods suffer from slow inference speeds and hallucinations. Additionally, the low-level control skills require manual expansion.

Given the limitations of LLM/VLM agents and reward shaping, LLM-guided RL emerges as a promising research area.
\cite{zhao2024large} utilizes the LLM's action probability and a UCT term as heuristic components, specifically: $\hat{\pi}(a|h)\frac{\sqrt{N(h)}}{N(h,a)+1}$, to influence the Q value.
\cite{du2023guiding} expands the observation input using an LLM and solely relies on the LLM's environmental analysis as the reward function.
\cite{wang2024rl} implements automated preference labeling through a VLM. These studies, while innovative in designing LLM-based reward heuristics, introduce performance biases and are susceptible to hallucinations.

To tackle these challenges, we propose a novel LLM-guided RL framework called \textbf{LLM-guided Q-learning} that introduces a heuristic term, $\mathbf{\hval}$, to the Q-function, expressed as: $\mathbf{\hat{\qval}}= \mathbf{\qval} + \mathbf{\hval}$. By employing the LLM generated heuristic values to modulate the values of the Q-function, we implicitly induce the desired policy. We use Table \ref{tab:algorithm_features} to display the capabilities of LLM-guided Q-learning compared to other popular frameworks.

In our theoretical analysis, we demonstrate that the heuristic term can help reduce suboptimality and discuss the impact of hallucinations from the perspectives of overestimation and underestimation. We then prove that the reshaped Q function can eventually converge to $\qval^*_D$ and provide the sample complexity for the reshaped Q function's convergence to $\qval^*_D$.

Experimentally,We conducted two experiments to determine if LLM-guided Q-learning can improve sample efficiency and whether the heuristic term introduces bias to the agent's converged performance. Results indicate that our framework is general and robust, as it does not require tuning hyperparameters and can adapt to different task settings. Furthermore, it improves sample efficiency by preventing invalid exploration and accelerates convergence.

The main contributions of this paper are: 
\begin{enumerate}
    \item It combines the advantages of both reward shaping techniques and the LLM/VLM Agent framework to improve sample efficiency.
    \item It transforms the impact of inaccurate or hallucinatory guidance into the cost of exploration.
    \item It supports online correction and can interact with human feedback.
\end{enumerate}

\begin{table}[h]

\caption{Comparison of Frameworks by Features. It shows that our framework combines the advantages of both reward shaping and LLM agents without introducing bias.}

\label{tab:algorithm_features}
\centering
\footnotesize
\begin{tabular}{l|c|c|c|c}
\hline
 & \makecell{Potential Reward\\Heuristic} & \makecell{Non-Potential\\ Reward Heuristic } & \makecell{LLM guided\\Q Heuristic} & \makecell{LLM/VLM \\agent} \\ \hline
Improve sample efficiency& $\checkmark$ & $\checkmark$ & $\checkmark$ & $\checkmark$   \\
Unbiased Agent Training &$\checkmark$& - & $\checkmark$ &-\\
Inaccurate Heuristic Adaptation &- & -  &  $\checkmark$ & $\checkmark$ \\
Interactive Training Correction & -&- & $\checkmark$ & $\checkmark$ \\
Real-Time Inference &$\checkmark$ &$\checkmark$ & $\checkmark$ & -\\
No Hallucination &$\checkmark$&$\checkmark$&$\checkmark$& -  \\
Complex Task Support &-&$\checkmark$&$\checkmark$& - \\ 
\hline
\end{tabular}

\end{table}

\vspace{-4pt}

\section{Related Work}

\subsection{Reward Shaping}
Reward shaping, initially introduced by \cite{rewardshaping_ang_wu}, is a crucial technique in reinforcement learning (RL) that modifies reward functions to improve agent learning efficiency. 
This approach has been further refined to accommodate various learning contexts and challenges. For example, Inverse RL \cite{ziebart2008maximum,wulfmeier2015maximum,finn2016guided} and Preference RL\cite{christiano2017deep,ibarz2018reward,lee2021pebble,park2022surf}, or reward heuristic: unsupervised auxiliary tasks reward \cite{jaderberg2016reinforcement}; 
count-based reward heuristics \cite{bellemare2016unifying,ostrovski2017count}; 
Self-supervised prediction errors as reward heuristics \cite{pathak2017curiosity,stadie2015incentivizing,oudeyer2007intrinsic}. 
As explained in \ref{sec:intro}, non-potential-based reward shaping introduces biased performance, and potential-based reward shaping can not provide guidance on actions. Furthermore, the reward shaping technique suffers from inaccurate guidance.
\subsection{LLM\textbackslash VLM Agent}
With limited samples and constraints regulating responses, LLMs/VLMs can achieve few-shot or even zero-shot learning in various contexts, as demonstrated by works such as Voyager \cite{wang2023voyager}, ReAct \cite{yao2022react}, SwiftSage \cite{lin2024swiftsage}, and SuspiciousAgent \cite{guo2023suspicion}.

In the field of robotics, VIMA \cite{jiang2022vima} employs multimodal learning to enhance agents' comprehension capabilities. Additionally, the use of LLMs for high-level control is becoming a trend in control tasks \cite{shi2024yell,liu2023interactive,ouyang2024long}.

In web search, interactive agents \cite{gur2023real,shaw2024pixels,zhou2023webarena} can be constructed using LLMs/VLMs. Moreover, to mitigate the impact of hallucinations, additional frameworks have been developed. For instance, decision reconsideration \cite{yao2024tree,long2023large}, self-correction \cite{shinn2023reflexion,kim2024language}, and observation summarization \cite{sridhar2023hierarchical} have been proposed to address this issue.
\subsection{LLM-enhanced RL}
Relying on the understanding and generation capabilities of large models, LLM-enhanced RL has become a popular field\cite{du2023guiding,carta2023grounding}. 
Researchers have investigated the diverse roles of large models within reinforcement learning (RL) architectures, including their application in reward design \cite{kwon2023reward,wu2024read,carta2023grounding,chu2023accelerating,yu2023language,ma2023eureka}, 
information processing \cite{paischer2022history,paischer2024semantic,radford2021learning}, 
as a policy generator, 
and as a generator within large language models (LLMs)\cite{chen2021decision,micheli2022transformers,robine2023transformer,chen2022transdreamer}.
While existing work on LLM-assisted reward design reduces design complexity, it either introduces bias into the original Markov Decision Process (MDP) or fails to offer adequate guidance for complex tasks.

\section{LLM-guided Q-learning}

\begin{wrapfigure}{r}{0.5\textwidth}
    \centering
    \includegraphics[width=\linewidth,trim=0.5cm 16.5cm 6.5cm 1cm,clip]{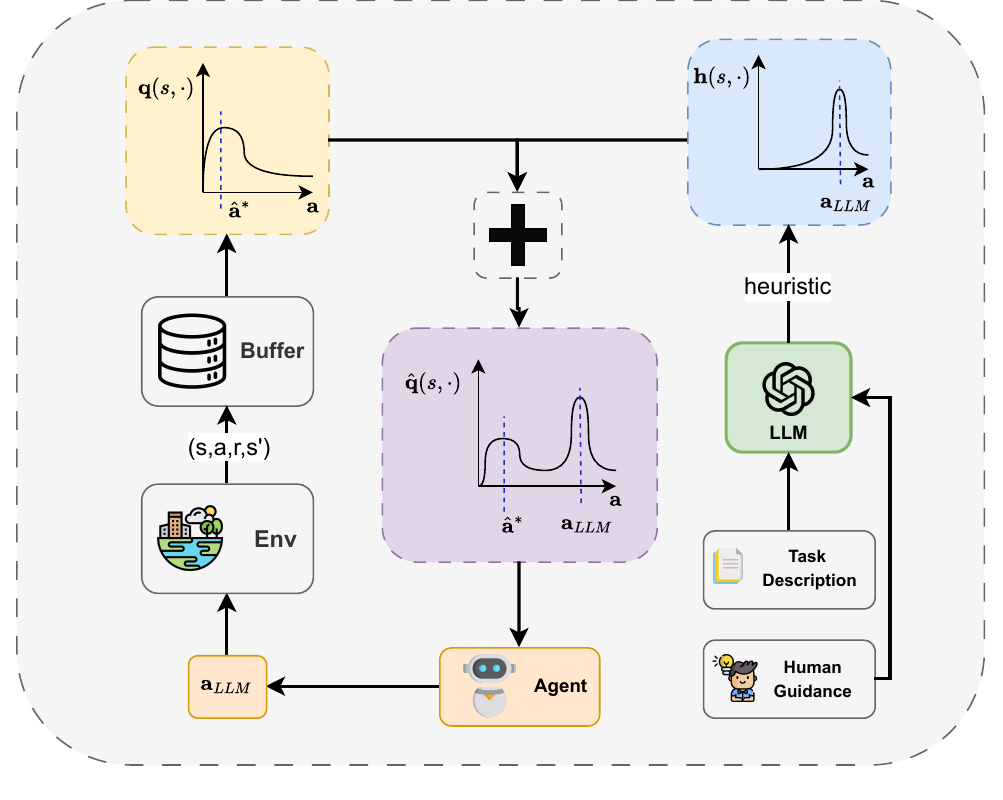}
    \caption{Framework of LLM-guided Q-learning. The update sources for Q-learning include both collected experience and LLM-generated heuristics. 
    }
    
    \label{fig:lql_framework}
    \vspace{-10pt}
\end{wrapfigure}
Non-potential-based reward designs can bias agent performance, while potential-based designs pose implementation challenges as they do not modify rewards at the action level or in terminal states. Given these challenges, we propose a novel training framework: LLM-guided Q-learning.
As shown in Figure \ref{fig:lql_framework}, the LLM can provide heuristic values at any iteration step to influence the Q function, which in turn affects the actions of the derived policy in subsequent iterations.
It represents a comprehensive training framework aimed at enhancing the sample efficiency of Q-learning algorithms. This approach can adapt to imprecise heuristics, recovering to original performance levels within a finite number of training steps after receiving incorrect guidance. Additionally, it supports online corrections and avoids biased performance, ensuring reliable functionality in dynamic learning environments.

In this section, we will discuss popular heuristic terms used in Q-learning and introduce novel heuristic terms generated by Large Language Models (LLMs). Additionally, we will introduce two practical implementations designed to effectively deploy this framework.
\subsection{Heuristic Q-learning Framework}
In the Q-learning framework, an experience buffer $D$ is used to store transitions from the Markov Decision Process (MDP), supporting both online and offline training. Our proposed framework utilizes both an experience buffer and heuristic values provided by a Large Language Model (LLM) to form a better Q-function. The general form is given by:
$$\hat{\qval}^{k+1}(s,a) = \qval^k(s,a) + \alpha\bigl(\lceil \overbrace{\erew(\sa) + \gamma\sum_{s' \in D(\sa)}\trnsd(s'|\sa)\max_a \qval^k(s,\cdot)}^{\text{target}} + 
\overbrace{\mathbf{h}(s,a)}^{\text{heuristic}} \rceil\bigr)$$
Here, $\mathbf{h}$ : $S \times A \xrightarrow{} \mathbb{R}$ is the heuristic term that adjusts the value of the target. The truncation operator $\lceil \cdot \rceil$  prevents potential Q-function divergence\cite{chen2023target_truncation} by ensuring that the input does not exceed $\frac{R_{\max}}{1-\gamma}$. The Q value is a high-level representation of both the environment and the policy of an agent. It encapsulates key elements such as rewards $r$, transition probabilities $P$, states $s$, actions $a$, and the policy $\pi$, thereby integrating the environmental dynamics and the policy under evaluation.
Changes in any of these components directly influence the Q values associated with different actions. 
Specifically, the term $\mathbf{h}$ can take many forms. Table \ref{tab:H_charts} lists some popular heuristic terms as well as new heuristic terms generated by Large Language Models (LLMs):
\begin{table}[h]
\caption{Heuristic Terms and Their Description.}
\label{tab:H_charts}
    \centering
    \begin{tabular}{|>{\raggedright\arraybackslash}m{0.25\textwidth}|>{\raggedright\arraybackslash}m{0.65\textwidth}|}
    
    \hline
    \textbf{Expression} & \textbf{Description} \\ \hline
    \( -\log \pi(a|s) \) \cite{sac} & Negative log-probability of action \(a\) given state \(s\), used in the Soft Actor-Critic algorithm. \\ \hline
    \( \sum_{t=t'}^{T} \gamma^t r_t \)\cite{geng2023improving,lin2018episodic} & Cumulative return of an episode, summed over rewards \( r_t \) at each time step \( t \), discounted by factor \( \gamma \). \\ \hline
    
    \(  c_{\text{uct}} \sqrt{\frac{\log N(s)}{N(s, a)}} \) \cite{kocsis2006bandit_mcts} & UCT heuristic for balancing exploration and exploitation in Monte Carlo Tree Search(MCTS)\cite{kocsis2006bandit_mcts,browne2012survey_mcts,coulom2006efficient_mcts}, where \( c_{\text{uct}} \) is a tunable parameter for exploration emphasis. \\ \hline 

    \(    c \sqrt{\frac{H^2\log(HSA/\delta)}{N_h(s,a) \vee 1}} \cite{rashidinejad2021bridging} \) &  penalty term in the VI-LCB algorithm for finite-horizon MDPs. \\ \hline
    \( \mathcal{R}(G(p(s,a))) \) \cite{kwon2023reward}\footnotemark & A generative model outputs the heuristic value directly, based on input prompt \( p \), which includes descriptions of state \( s \), and action \( a \). \\ \hline
    \( \mathcal{T}(G(p))(s,a) \)  & A generative model first generates executable code from prompt \( p \) that acts as a heuristic function, then evaluates it with inputs \( s \) and \( a \). \\ \hline
    \end{tabular}
    
\end{table}
\footnotetext{The expression is originally used for generating rewards, but can also serve as a Q-value heuristic.}

Integrating Q-functions with different heuristic functions leads to agents displaying distinct behavioral styles.
\paragraph{Limitations of Action-Bonus Heuristics}
Current popular heuristic Q-learning approaches, such as MCTS and SAC, 
utilize action-bonus heuristics like maximum entropy terms or count bonus terms to encourage exploration.
These methods are essentially neighbourhood search algorithms. 
For example, in SAC, denote $\hat{a}^*$ as the optimal action derived from the learned Q-function in state $s$, and $\hat{a}^*+ \epsilon$ is close to $\hat{a}^*$. 
We can infer that the value $\hat{\qval}(s,\hat{a^*}+\epsilon)$ is further increased by the entropy term because $\pi(\hat{a^*}+\epsilon | s ) < \pi(\hat{a^*} | s )$.
When the exported policy from the current learned Q function deviates significantly from the ground truth optimal policy, the agent wastes a tremendous number of samples exploring areas that stray from the optimal trajectory, leading to sample inefficiency. 
This limitation also explains why the same algorithm performs differently even in the same environment.

Recent studies\cite{wang2024rl,kwon2023reward,hu2023language,yu2023language,ma2023eureka} in LLM-enhanced reinforcement learning (RL) have used LLMs to generate rewards or reward functions, which can bias agent performance or slow down training due to the time required for generating reward signals on a step-by-step basis. 
By using LLMs to generate heuristic Q-values instead, we overcome issues associated with reward shaping. This approach enables action-level heuristic values and avoids biasing agent capabilities after convergence.

\subsection{Algorithm Implementation}
\label{sec:implementation}

In the implementation of LLM-guided Q-learning, we provide two kinds of algorithms: Offline-Guidance Q-Learning and Online-Guidance Q-Learning. The Offline-Guidance Q-Learning is a special case of Online-Guidance Q-Learning that only provides guidance at iteration step $k=0$.
We extend the TD3 algorithm to LLM-TD3, which can accept heuristic values from LLM. To learn from the experience buffer, we employ an L2 loss to approximate the buffer's Q-values, defined as:

\begin{equation}
    \label{eq:qhatloss}
    L_{major}(\theta) = E_{(s,a,r,s') \sim D}[(\lceil r+\gamma \min_i \hat{\qval}_{\theta_i}(s',\hat{a}+\epsilon) \rceil-   \hat{\qval}_\theta(s,a) )^2]
\end{equation}

In Equation \ref{eq:qhatloss}, we apply the truncation operator $\lceil \cdot \rceil$ to the target Q value to prevent potential divergence examples while ensuring that it is consistent with our theoretical analysis assumptions. The pseudo-code of offline guidance Q learning is displayed at Algorithm \ref{alg:offline_q_learning}: 

\begin{algorithm}
\caption{Offline-Guidance Q-Learning Algorithm}
\label{alg:offline_q_learning}
\begin{algorithmic}[1]
\State \textbf{Inputs:} Local MDP $D$, Large Language Model $G$,prompt $p$,
\State \textbf{Initialization:} Initialize the heuristic Q buffer: $D(G(p))$,initialize actor-critic $(\mu_\phi,Q_{\theta_1},Q_{\theta_2})$ and target actor-critic $(\mu_{\phi^{'}},Q_{\theta^{'}_1},Q_{\theta^{'}_2})$ 
\State \textbf{Generate Q-buffer:} $D_g \leftarrow \{ (s_i, a_i,Q_i) \mid (s_i, a_i,Q_i) \in D(G(p)), \, i = 1, 2, \ldots, n \}$
\State \textbf{Q-Bootstrapping:} $\theta =\theta -\alpha \nabla_\theta L_{bootstrap}$ using  eq \ref{eq:bootstrap_loss}
    \For{iteration $t' \in T=1,2,3...$}
        \State Sample $(s,a,r,s')$ from Env
        \State $D \leftarrow D \cup (s,a,r,s')$        
        \State Sample N transitions $(s,a,r,s')$ from $D$
        \State $\hat{a} \leftarrow \mu(s') + \epsilon, \epsilon \sim \text{clip}(\mathcal{N}(0,\sigma),-c,c)$
        \State $y(s') \leftarrow \lceil r+\gamma \min_{i=1,2}\qval_{\theta_i}(s',\hat{a})  \rceil $
        \State Update critics $\theta_i  \leftarrow \arg \min_{\theta_i}  L_{major}(\theta_i)$
        
        \If{ $t' \mod d$}
            \State Update actor $\phi \leftarrow \arg \min_\phi L_{actor}(\phi)$ 
            \State Update target networks:
            \State $\theta_{i}^{'} \leftarrow \tau\theta_i + (1-\tau)\theta_{i}^{'}$
            \State $\phi^{'} \leftarrow \tau\phi + (1-\tau)\phi^{'} $
        \EndIf
        \EndFor
\State \textbf{end while}
\end{algorithmic}
\end{algorithm}

As shown in Algorithm \ref{alg:offline_q_learning}, two additional step is added to the TD3 algorithm: Q-buffer Generation and Q bootstrapping. For Q-buffer Generation, the input $p$ is the combination of task description, and predefined prompt. 
Task description contains  specific descriptions of state-actions and provides conditions for trajectory termination. 
Treating text descriptions as input to the large model can help it better generate Q values. 
In order to constrain the model's answers, we developed a general template, which can be combined with the description of any task to form the input of the large model. Its main task is to guide the LLM to generate an executable python function, which will return two set of $(s,a,Q)$ pairs.

In the step "Generate Q-Buffer," the general template and task description serve as input to an LLM, which produces executable Python code. This code returns a set of (s, a, Q) pairs, collectively referred to as 
$D_g$, and is defined as follows:
$$D_g := \{ (s_i, a_i,Q_i) \mid (s_i, a_i,Q_i) \in D(G(p)), \, i = 1, 2, \ldots, n \}$$
Note that we use $D(\cdot)$ to denote the transformation of LLM-generated text into a set of $(s,a,Q)$ pairs.
In the step "Q-Bootstrapping", we employ the following loss function to bootstrap $\hat{\qval}$:
\begin{equation}
\label{eq:bootstrap_loss}
     L_{bootstrap}(\theta) = E_{(s_i,a_i,Q_i) \sim D_g}{(Q_i - \hat{\qval}_{\theta}(s_i,a_i))^2}
\end{equation}

Since GPT-4's ability to predict Q-values decreases as the complexity of the environment increases, and Offline Guidance cannot predict in advance all potential problems encountered by the agent during the training process, we propose Online-Guidance Q-Learning, a method that enables the agent to receive guidance at any training step from human feedback.
The corresponding loss function for Online-Guidance Q-learning is presented in Equation \ref{eq:online_hq}, and the pseudo-code for Online-Guidance is shown in Appendix \ref{sec:online_guidance}:
\begin{equation}
    \label{eq:online_hq}
    L_{online}(\theta) =  L_{major}(\theta) + E_{(s_i,a_i,Q_i) \sim D(G(p))}[(\hat{\qval}_{\theta}(s_i,a_i) - Q_i)^2]
\end{equation}

The Online-Guidance algorithm detects whether there is guidance from external sources at each training step. When the agent engages in ineffective exploration, humans provide guidance information, which the language model then transforms into $(s,a,Q)$ pairs. The algorithm then uses the loss function in Equation \ref{eq:online_hq} to train the Q-function. A visualization for Online-Guidance can be seen in Figure \ref{fig:online_lql}.
\begin{figure}[h]
    \centering
    \includegraphics[angle=90,trim=6.5cm 0cm 5.5cm 2.3cm,clip, scale=0.55]{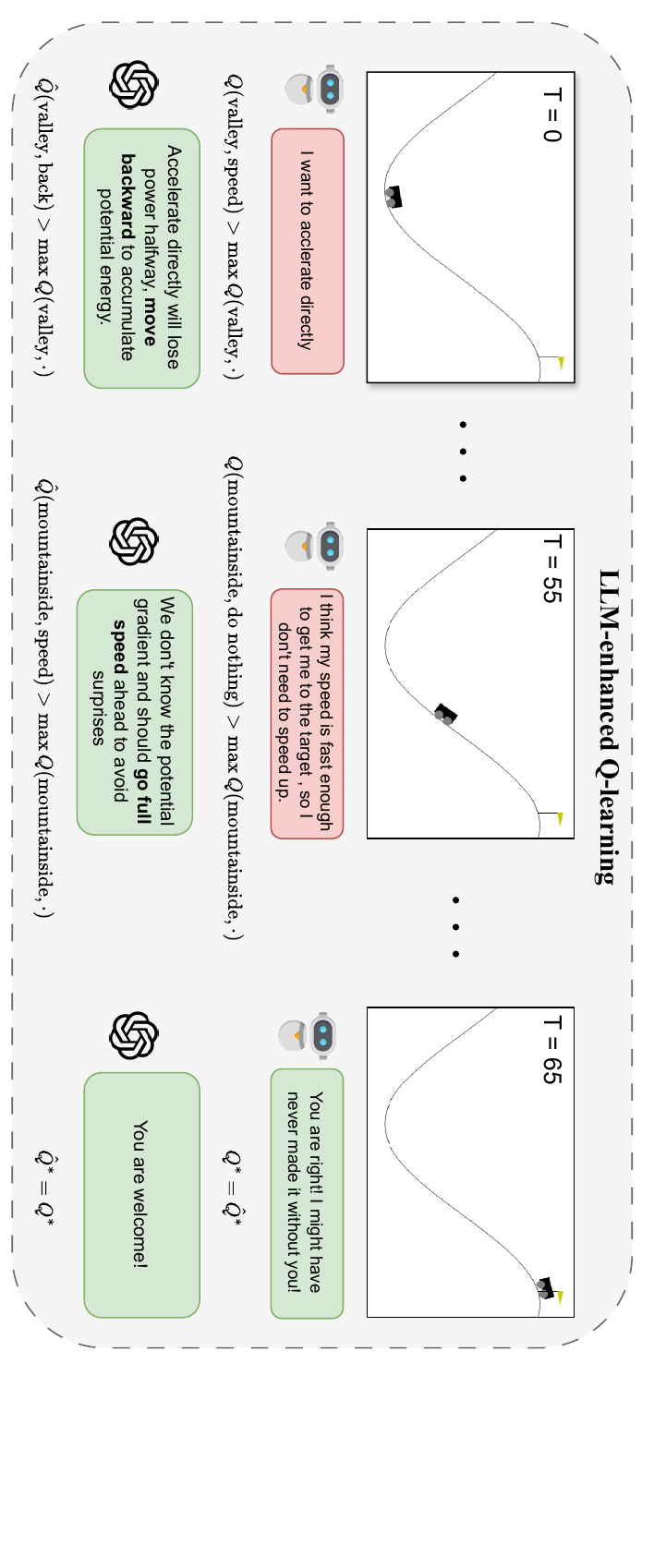}
    \caption{Overview of the Online Guidance Q-learning framework. External guidance can be provided at any training step, offering heuristic Q-values that influence policy decisions and improve sample efficiency.}
    \vspace{-5pt}
    \label{fig:online_lql}
\end{figure}

\section{Theoretical Analysis}

In this section, we provide suboptimality analysis at the \( s, a \) level, as well as an analysis of how hallucination affects Q-values. This elucidates how our framework accelerates algorithm convergence. 
We also present a sampling complexity analysis of the convergence of an arbitrary bounded, reshaped Q-function to the optimal Q-function of the collected MDPs.
\subsection{Suboptimality Analysis}
\label{sec:suboptimality_analysis}
By defining $a^* = \arg\max_a \qval^*(s,a)$, $a^{\pi_D^*} = \arg\max_a \qval^*_D(s,a)$, and $a^\pi = \max_a \pi(a|s)$ for any random policy $\pi$, we have:

\begin{align*}
\qval^*(s,a^*)-(\qval^{\pi_D^*})^*(s,a^{\pi_D^*}) \le
\underbrace{
\inf_{\pi}\bigl( 
\overbrace{
\qval^{*}(s,a^{\pi^*}) -\hat{\qval}_D^{*}(s,a^{\pi^*}) 
}^{\textsc{(a1)}}+ 
\overbrace{
\hat{\qval}_D^*(s,a^{\pi^*})-(\hat{\qval}_D^{\pi})^*(s,a^{\pi})\bigr)}^{\textsc{(a2)}}
}_{\textsc{(a)}}+ \\
\underbrace{ \sup_\pi \bigl((\hat{\qval}_D^\pi)^*(s,a^\pi)-(\qval^\pi)^*(s,a^\pi)\bigr)}_{\textsc{(b)}}
\end{align*}

\begin{proof}
See proof at \ref{sec:proof_suboptimality}
\end{proof}

The first segment, Item $\textsc{(a)}$, evaluates the divergence between the Q-function derived from the sampled MDP $D$ and the optimal Q-function. 
For any state $s$, sub-item (A1) assesses overestimation by using the optimal action $a^*$ to measure the difference between the Q-value of the optimal action in $D$ and the Q-value of the actual optimal action.
Sub-item (A2) analyzes the difference in Q-values under different policies.
Item $\textsc{(a)}$ assesses whether the buffer MDP $D$ effectively captures optimal trajectories. If so, minimizing Item $\textsc{(a)}$ is viable as it confirms that the optimal action in the buffer MDP aligns with the true optimal action.
Item $\textsc{(b)}$ measures the differences in Q-values under different policies. If $D$ does not contain the optimal trajectory, this gap will persist throughout training, and thus the heuristic term is needed to bridge this gap and guide the policy toward sampling the optimal trajectory.

\subsection{Impact of Hallucination}
\label{sec:impact_hallucination}
The hallucination problem in large language models (LLMs) manifests in two primary ways: the generation of verbose responses that fail to meet specified requirements and the production of inaccurate Q-values. These imprecise Q-values can be further categorized into overestimation and underestimation.

Overestimation has been clearly discussed in TD3\cite{td3}. The main point of overestimation is that the error $\epsilon$ on the target will influence the TD update and thus impact all the Q values of previous state-action pairs. 

The impact of underestimation bias can be divided into two situations: the underestimation of non-optimal actions and the underestimation of optimal actions.
We define $a^* := \arg \max_a \qval^*(s,a)$, and $\hat{a}^* = \arg \max_a \hat{\qval}(s,a)$, and have:

\paragraph{Underestimation on Non-Optimal Actions}
Assume that for some non-optimal action $a$,
$\hat{\qval}(s,a) < \qval^*(s,a) \le \qval^*(s,a^*)$. This means that the learned Q function underestimates non-optimal action $a$. 
According to the definition of the greedy sampling policy $\mu$, this underestimation setting will not impact the policy, and only $\hat{a}^*$ will.
\paragraph{Underestimation on Optimal Actions}

We can separate underestimation on optimal actions into two cases. Case 1: $\hat{a}^* = a^*$, and Case 2 : $\hat{a}^* \ne a^* $.
Both cases satisfies: $\hat{\qval}(s,\hat{a}^*) < \qval^*(s,a^*)$

Case 1:
The derived policy will not be influenced by the underestimation, because the Q function and the policy function have a bijective relations. In Case 1, the relationship between $\qval^*$ and $\hat{\qval}$ can be expressed as $\qval^*(s,a^*) > \hat{\qval}(s,a^*) > \mathrm{secmax}_a \hat{\qval}(s,a)$ and the exported policy is the same , which won't influence the performance of the agent. This explains why CQL \cite{kumar2020conservative_cql} is still successful even with pessimistic Q-values.

Case 2:
With the assumption  $\hat{a}^* \ne a^* $, we can derive that :$\hat{\qval}(s,a^*) < \max_a \hat{\qval}(s,a) < \qval^*(s,a^*)$
. This kind of underestimation will impact the performance of the agent since the optimal action of the learned Q function isn't the ground truth optimal action.
To solve this problem, we either improve the value of $a^*$ or decrease the value of other non-optimal actions.
However, decreasing the value of all other non-optimal actions is impossible ,as it's computationally complex. Another feasible method to make $a^*$ optimal in the learned Q function is to improve by sampling a high discounted return trajectory that includes $a^*$, or use a heuristic term to force $a^*$ to be optimal in $\hat{\qval}(s,\cdot)$.

The analysis results show that underestimation of non-optimal action will not impact the agent's performance and can help avoid ineffective exploration. There is no strict limitation on the Q-value estimation of the optimal action, thus inaccurate guidance will not impact the performance of the agent according to Case 1. When $D$ includes the optimal trajectory , $\hat{\qval}$ will converge to $\qval^*$ very soon.

\subsection{Convergence Analysis}

\begin{theorem}[Contraction and Equivalence of $\hat{\qval}$]
\label{theorem:Q_contraction_Q_equal}
Let $\hat{\qval}$ be a contraction mapping defined in the metrics space $(\mathcal{X},\|\cdot\|_{\infty})$, i.e,
$$\|   \bell_{D}(\hat{\qval}) -\bell_{D}(\hat{\qval}')  \|_{\infty} \le \gamma \|\hat{\qval} - \hat{\qval}' \|_{\infty}  $$, where $\bell_{D}$ is the Bellman operator for the sampled MDP \(D\) and \(\gamma\) is the discount factor.

Since both $\hat{\qval}$ and $\qval$ are updated on the same MDP,
we have the following equation:
$$\hat{\qval}_D^{*} = \qval^{*}_D$$
\label{thm:qhat_is_contraction}
\end{theorem}
\begin{proof}
See Appendix \ref{sec:proof_qhat_contraction}
\end{proof}
Assume that LLM provides a relatively high heuristic Q value for state-action pair $\langle s',a'\rangle$ at training step k. and it satisfies :$ \max_{\tau} G_k(\tau) < \qval^k(s',a') + \hval^k(s',a')$. 
This inequality can also be expressed as:
$$\max_{s,a} \qval(s,a) < \qval^k(s',a') + \hval^k(s',a')$$
As the $\hat{\qval}$ is updated using next state's Q value, then we have the following inequality:

    $$\hat{\qval}^k(s',a') > \max_{s,a} \qval^k(s,a) \ge \max_a \qval^k(s',a)$$

This makes $a'$ the optimal action in $\hat{\qval}(s',a')$ at training step $k$. In online training process, $\mu(s')$ will directly execute $a'$ in the next following training step, allowing the agent to avoid exploring non-optimal trajectories.
In the limit, according to \cite{NIPS2010_091d584f,td3} and Theorem \ref{thm:qhat_is_contraction}, $\hat{\qval}$ converges to the optimal Q function. We now outline the theoretical upper limit on the sample complexity required for $\hat{\qval}$ to converge to $\qval^*_D$ for any MDP $D$:

\begin{theorem}[Convergence Sample Complexity]
\label{thm:sample complexity}
The sample complexity $n$ required for $\hat{\qval}$ to converge to the optimal fixed-point $\qval^*_D$ with probability $1-\delta$ is:
$$n > \mathcal{O}\left( \frac{|S|^2}{2\epsilon^2}\ln{\frac{2|S \times A|}{\delta}} \right)$$

\end{theorem}

\begin{proof}
See proof at \ref{sec:proof_sample_complexity}.    
\end{proof}

Theorem \ref{thm:sample complexity} indicates that the sample complexity depends on the size of the state and action spaces. Additionally, it asserts that the hallucination guidance provided by large language models (LLMs) will be eliminated within a finite number of steps.

\section{Experiment}
\label{sec:exp}

We evaluate our proposed algorithm across eight diverse Gymnasium environments, spanning from the straightforward MountainCar to the highly complex Humanoid Run. The LLM heuristic term is integrated with TD3 \cite{td3}, forming \textbf{LLM-TD3}. We conduct two types of evaluations: a convergence speed test and an adaptability test, aiming to assess whether the LLM heuristic term accelerates learning and ensures that the agent can recover from incorrect guidance. Specifically, we compare LLM-TD3 with four widely used baselines, assessing convergence speed to determine each algorithm's sample efficiency.

\paragraph{Setting} We choose three simple tasks from classical control: \textit{cartpole}, \textit{mountaincarcontinuous}, and \textit{pendulum}, along with five complex tasks from Mujoco: \textit{ant}, \textit{halfcheetah}, \textit{hopper}, \textit{humanoid}, and \textit{walker2d}. \footnote{Please see \href{https://gymnasium.farama.org/}{gymnasium.farama.org} for more details.} 

\begin{table}[h]
\caption{Comparative Analysis of Convergence Speed and Guidance Impact on LLM-guided Q-learning Across Environments}
\setlength\extrarowheight{-3pt} 
\setlength\tabcolsep{3pt}
\centering\scriptsize

\begin{tabular}{@{}ccccccccc@{}}

\toprule
\multicolumn{2}{c}{Guidance Count} &  &       & \multicolumn{5}{c}{Algorithms}   \\ 
\cmidrule{1-2} \cmidrule{5-9}
good Q & bad Q & |S|+|A| & Env & PPO & SAC &  DDPG & TD3 & LLM-TD3 \\ \midrule
3 & 1 & 2+1 & MountainCarContinuous & \textit{Failed} &\textit{Failed} & \textit{Failed} & \textit{Failed} & \textbf{140k}\\
1 & 2 & 3+1 & InvertedPendulum &57k &28k &  153k & 30k & \textbf{25k} \\
2 & 2 & 4+1 & Pendulum &86k & 7k &6k &5k  & \textbf{4k} \\
1 & 1 & 11+3 & Hopper &\textit{Failed} & 634k &\textit{Failed} &474k & \textbf{269k}\\
1 & 1 & 17+6 & Walker2d & >1000k&524k &  \textit{Failed} & 569k & \textbf{514k}  \\
0 & 2 & 17+6 & HalfCheetah & \textit{Failed}&>1000k& 716k & 844k & \textbf{689k}\\
0 & 2 & 27+8 & Ant &\textit{Failed}  &>1000k & \textit{Failed}& 444k & \textbf{389k} \\
1 & 1 & 376+17 & Humanoid &\textit{Failed} &\textbf{444k} & \textit{Failed}& 664k & 594k\\

 \bottomrule
\end{tabular}

\label{tab:convergence_speed}
\end{table}

In our algorithm implementation, we employ GPT-4 to provide heuristic values based on task descriptions and human feedback. In the convergence speed test, we supply heuristic values at iteration step $t'=0$. Additionally, the hyperparameters of LLM-TD3 are fixed for each task.

\paragraph{Metrics} We evaluate each algorithm's sample efficiency through convergence speed. To account for the randomness of the environment and initialization, each algorithm is run 10 times, and the average performance of these runs is used for evaluation. After every 1e3 steps sampled by the agent, we verify the algorithm's performance by requiring the agent to run 10 episodes without updating, using the average score of these episodes as the evaluation value. An agent is considered converged when its performance reaches 80\% of the peak performance. If an algorithm fails to reach 80\% performance, it will be marked as "Failed."

\paragraph{Baseline} In terms of baseline selection, we compare our model with several of the most popular algorithms, namely: PPO\cite{schulman2017proximal_ppo,huang2022cleanrl}, SAC\cite{sac}, TD3, and DDPG\cite{lillicrap2015continuous_ddpg}.
\paragraph{Results} 
The convergence speed test results are presented in Table \ref{tab:convergence_speed}. The "Guidance Count" column shows the number of $(s, a, Q)$ pairs provided by the LLM based on the task description, with 1-4 different guidance sets per task. In our algorithm, these pairs are categorized as either good Q pairs (high-value $(s, a)$ pairs) or bad Q pairs (zero or negative-value $(s, a)$ pairs).
Our results demonstrate that:
\textbf{(1) Avoid Unnecessary Exploration:} The provided heuristic values reduce unnecessary exploration, thereby improving sample efficiency.
\textbf{(2) Robust Performance:} Our algorithm is robust and performs well without the need for hyperparameter tuning.
\textbf{(3) Generalizability:} LLM-guided Q-learning is generalizable and can adapt to different tasks, provided the LLM can understand the content or there is a rich description available.
\textbf{(4) Incorporating Human Feedback:} Our framework transforms general feedback, such as "do not be lazy" into (s, a) restrictions, enhancing sample efficiency without biasing the final performance.

\subsection{Adaptability Test}

To test the adaptability of our framework, we conducted experiments on the \textit{pendulum-v1} environment to evaluate whether the agent can recover from incorrect heuristic guidance at any training stage. Specifically, we provided incorrect heuristic values at three different periods: (1) the beginning of training, (2) during training, and (3) after convergence (4) every 10k steps. The incorrect heuristic encourages the agent to swing the pendulum downward and maintain no speed, which is contrary to the environment's reward structure.

\begin{figure}
    \centering
    \includegraphics[width=1\linewidth]{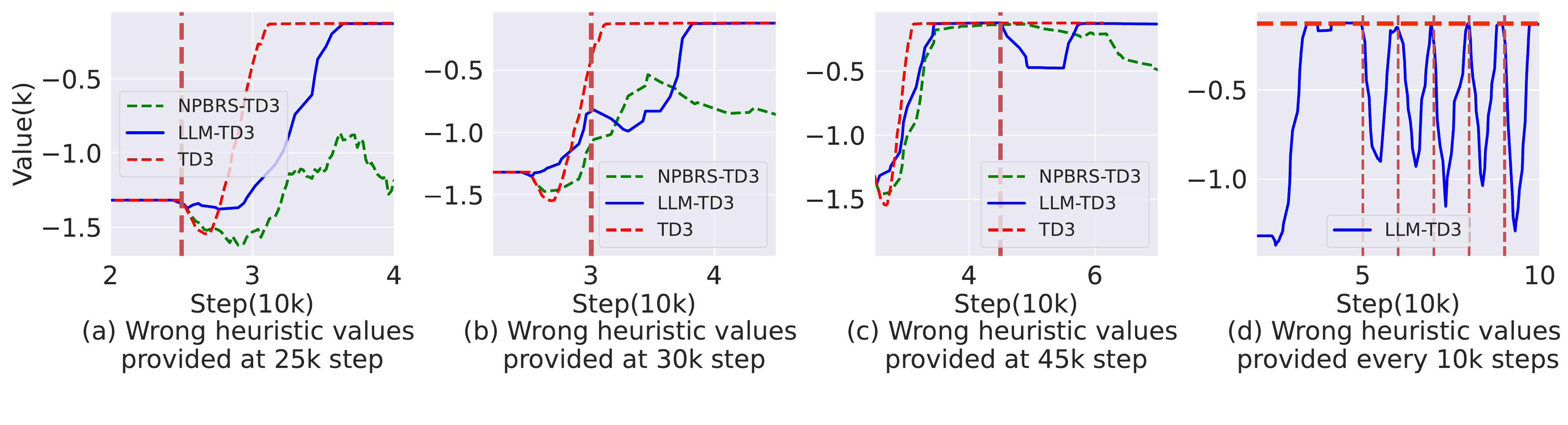}
     \caption{Results of the adaptability test at different periods. The red horizontal line in the fourth figure represents optimal performance. The results indicate that after receiving incorrect heuristic values, our algorithm quickly recovers to its original performance levels.}
   
    \label{fig:adaptability}
\end{figure}

For comparison, we also implemented a non-potential reward shaping method. As shown in Figure \ref{fig:adaptability}, the reward shaping enhanced TD3 algorithm failed to recover to the original performance level after receiving the incorrect heuristic. In contrast, the LLM-guided Q-learning approach successfully recovered, demonstrating the robustness and adaptability of our framework.

\section{Conclusion}

We introduced LLM-guided Q-learning, utilizing Large Language Models (LLMs) as heuristic value estimators to improve sample efficiency without biasing final performance. Our theoretical analysis and experiments highlight three key benefits:

\begin{enumerate}
    \item \textbf{Learning from Inaccurate Data}: Facilitates effective learning from imperfect data, incorporating guidance from generative models.
    \item \textbf{Immediate Q Value Impact}: Q-shaping directly alters the Q-value distribution, significantly affecting agent performance more immediately than reward shaping.
    \item \textbf{General and Robust}: Applicable to various tasks without task-specific hyperparameter tuning, demonstrating versatility and robustness.
\end{enumerate}

Our framework effectively addresses the challenges of using generative models to accelerate RL training, paving the way for the subsequent introduction of more powerful generative models.
\newpage
\bibliographystyle{unsrt}
\bibliography{reference2}

\begin{thebibliography}{10}

\bibitem{watkins1992q}
Christopher~JCH Watkins and Peter Dayan.
\newblock Q-learning.
\newblock {\em Machine learning}, 8:279--292, 1992.

\bibitem{ddpg}
Timothy~P Lillicrap, Jonathan~J Hunt, Alexander Pritzel, Nicolas Heess, Tom Erez, Yuval Tassa, David Silver, and Daan Wierstra.
\newblock Continuous control with deep reinforcement learning.
\newblock {\em arXiv preprint arXiv:1509.02971}, 2015.

\bibitem{sac}
Tuomas Haarnoja, Aurick Zhou, Kristian Hartikainen, George Tucker, Sehoon Ha, Jie Tan, Vikash Kumar, Henry Zhu, Abhishek Gupta, Pieter Abbeel, et~al.
\newblock Soft actor-critic algorithms and applications.
\newblock {\em arXiv preprint arXiv:1812.05905}, 2018.

\bibitem{td3}
Scott Fujimoto, Herke Hoof, and David Meger.
\newblock Addressing function approximation error in actor-critic methods.
\newblock In {\em International conference on machine learning}, pages 1587--1596. PMLR, 2018.

\bibitem{chen2023target_truncation}
Zaiwei Chen, John-Paul Clarke, and Siva~Theja Maguluri.
\newblock Target network and truncation overcome the deadly triad in-learning.
\newblock {\em SIAM Journal on Mathematics of Data Science}, 5(4):1078--1101, 2023.

\bibitem{cheng2021heuristic}
Ching-An Cheng, Andrey Kolobov, and Adith Swaminathan.
\newblock Heuristic-guided reinforcement learning.
\newblock {\em Advances in Neural Information Processing Systems}, 34:13550--13563, 2021.

\bibitem{Hu_Wang_Jia_Wang_Chen_Hao_Wu_Cheng_2019}
Youyou Hu, Weixun Wang, Hangtian Jia, Yixiang Wang, Yingfeng Chen, Jianye Hao, Feng Wu, and Fan Cheng.
\newblock Learning to utilize shaping rewards: A new approach of reward shaping.
\newblock {\em Neural Information Processing Systems,Neural Information Processing Systems}, Dec 2019.

\bibitem{brooks2024large}
Ethan Brooks, Logan Walls, Richard~L Lewis, and Satinder Singh.
\newblock Large language models can implement policy iteration.
\newblock {\em Advances in Neural Information Processing Systems}, 36, 2024.

\bibitem{wang2023voyager}
Guanzhi Wang, Yuqi Xie, Yunfan Jiang, Ajay Mandlekar, Chaowei Xiao, Yuke Zhu, Linxi Fan, and Anima Anandkumar.
\newblock Voyager: An open-ended embodied agent with large language models.
\newblock {\em arXiv preprint arXiv:2305.16291}, 2023.

\bibitem{shi2024yell}
Lucy~Xiaoyang Shi, Zheyuan Hu, Tony~Z Zhao, Archit Sharma, Karl Pertsch, Jianlan Luo, Sergey Levine, and Chelsea Finn.
\newblock Yell at your robot: Improving on-the-fly from language corrections.
\newblock {\em arXiv preprint arXiv:2403.12910}, 2024.

\bibitem{liu2023interactive}
Huihan Liu, Alice Chen, Yuke Zhu, Adith Swaminathan, Andrey Kolobov, and Ching-An Cheng.
\newblock Interactive robot learning from verbal correction.
\newblock {\em arXiv preprint arXiv:2310.17555}, 2023.

\bibitem{ouyang2024long}
Yutao Ouyang, Jinhan Li, Yunfei Li, Zhongyu Li, Chao Yu, Koushil Sreenath, and Yi~Wu.
\newblock Long-horizon locomotion and manipulation on a quadrupedal robot with large language models.
\newblock {\em arXiv preprint arXiv:2404.05291}, 2024.

\bibitem{zhao2024large}
Zirui Zhao, Wee~Sun Lee, and David Hsu.
\newblock Large language models as commonsense knowledge for large-scale task planning.
\newblock {\em Advances in Neural Information Processing Systems}, 36, 2024.

\bibitem{du2023guiding}
Yuqing Du, Olivia Watkins, Zihan Wang, C{\'e}dric Colas, Trevor Darrell, Pieter Abbeel, Abhishek Gupta, and Jacob Andreas.
\newblock Guiding pretraining in reinforcement learning with large language models.
\newblock In {\em International Conference on Machine Learning}, pages 8657--8677. PMLR, 2023.

\bibitem{wang2024rl}
Yufei Wang, Zhanyi Sun, Jesse Zhang, Zhou Xian, Erdem Biyik, David Held, and Zackory Erickson.
\newblock Rl-vlm-f: Reinforcement learning from vision language foundation model feedback.
\newblock {\em arXiv preprint arXiv:2402.03681}, 2024.

\bibitem{rewardshaping_ang_wu}
Andrew~Y Ng, Daishi Harada, and Stuart Russell.
\newblock Policy invariance under reward transformations: Theory and application to reward shaping.
\newblock In {\em Icml}, volume~99, pages 278--287, 1999.

\bibitem{ziebart2008maximum}
Brian~D Ziebart, Andrew~L Maas, J~Andrew Bagnell, Anind~K Dey, et~al.
\newblock Maximum entropy inverse reinforcement learning.
\newblock In {\em Aaai}, volume~8, pages 1433--1438. Chicago, IL, USA, 2008.

\bibitem{wulfmeier2015maximum}
Markus Wulfmeier, Peter Ondruska, and Ingmar Posner.
\newblock Maximum entropy deep inverse reinforcement learning.
\newblock {\em arXiv preprint arXiv:1507.04888}, 2015.

\bibitem{finn2016guided}
Chelsea Finn, Sergey Levine, and Pieter Abbeel.
\newblock Guided cost learning: Deep inverse optimal control via policy optimization.
\newblock In {\em International conference on machine learning}, pages 49--58. PMLR, 2016.

\bibitem{christiano2017deep}
Paul~F Christiano, Jan Leike, Tom Brown, Miljan Martic, Shane Legg, and Dario Amodei.
\newblock Deep reinforcement learning from human preferences.
\newblock {\em Advances in neural information processing systems}, 30, 2017.

\bibitem{ibarz2018reward}
Borja Ibarz, Jan Leike, Tobias Pohlen, Geoffrey Irving, Shane Legg, and Dario Amodei.
\newblock Reward learning from human preferences and demonstrations in atari.
\newblock {\em Advances in neural information processing systems}, 31, 2018.

\bibitem{lee2021pebble}
Kimin Lee, Laura Smith, and Pieter Abbeel.
\newblock Pebble: Feedback-efficient interactive reinforcement learning via relabeling experience and unsupervised pre-training.
\newblock {\em arXiv preprint arXiv:2106.05091}, 2021.

\bibitem{park2022surf}
Jongjin Park, Younggyo Seo, Jinwoo Shin, Honglak Lee, Pieter Abbeel, and Kimin Lee.
\newblock Surf: Semi-supervised reward learning with data augmentation for feedback-efficient preference-based reinforcement learning.
\newblock In {\em 10th International Conference on Learning Representations, ICLR 2022}. International Conference on Learning Representations, 2022.

\bibitem{jaderberg2016reinforcement}
Max Jaderberg, Volodymyr Mnih, Wojciech~Marian Czarnecki, Tom Schaul, Joel~Z Leibo, David Silver, and Koray Kavukcuoglu.
\newblock Reinforcement learning with unsupervised auxiliary tasks.
\newblock In {\em International Conference on Learning Representations}, 2016.

\bibitem{bellemare2016unifying}
Marc Bellemare, Sriram Srinivasan, Georg Ostrovski, Tom Schaul, David Saxton, and Remi Munos.
\newblock Unifying count-based exploration and intrinsic motivation.
\newblock {\em Advances in neural information processing systems}, 29, 2016.

\bibitem{ostrovski2017count}
Georg Ostrovski, Marc~G Bellemare, A{\"a}ron Oord, and R{\'e}mi Munos.
\newblock Count-based exploration with neural density models.
\newblock In {\em International conference on machine learning}, pages 2721--2730. PMLR, 2017.

\bibitem{pathak2017curiosity}
Deepak Pathak, Pulkit Agrawal, Alexei~A Efros, and Trevor Darrell.
\newblock Curiosity-driven exploration by self-supervised prediction.
\newblock In {\em International conference on machine learning}, pages 2778--2787. PMLR, 2017.

\bibitem{stadie2015incentivizing}
Bradly~C Stadie, Sergey Levine, and Pieter Abbeel.
\newblock Incentivizing exploration in reinforcement learning with deep predictive models.
\newblock {\em arXiv preprint arXiv:1507.00814}, 2015.

\bibitem{oudeyer2007intrinsic}
Pierre-Yves Oudeyer and Frederic Kaplan.
\newblock What is intrinsic motivation? a typology of computational approaches.
\newblock {\em Frontiers in neurorobotics}, 1:108, 2007.

\bibitem{yao2022react}
Shunyu Yao, Jeffrey Zhao, Dian Yu, Nan Du, Izhak Shafran, Karthik Narasimhan, and Yuan Cao.
\newblock React: Synergizing reasoning and acting in language models.
\newblock {\em arXiv preprint arXiv:2210.03629}, 2022.

\bibitem{lin2024swiftsage}
Bill~Yuchen Lin, Yicheng Fu, Karina Yang, Faeze Brahman, Shiyu Huang, Chandra Bhagavatula, Prithviraj Ammanabrolu, Yejin Choi, and Xiang Ren.
\newblock Swiftsage: A generative agent with fast and slow thinking for complex interactive tasks.
\newblock {\em Advances in Neural Information Processing Systems}, 36, 2024.

\bibitem{guo2023suspicion}
Jiaxian Guo, Bo~Yang, Paul Yoo, Bill~Yuchen Lin, Yusuke Iwasawa, and Yutaka Matsuo.
\newblock Suspicion-agent: Playing imperfect information games with theory of mind aware gpt-4.
\newblock {\em arXiv preprint arXiv:2309.17277}, 2023.

\bibitem{jiang2022vima}
Yunfan Jiang, Agrim Gupta, Zichen Zhang, Guanzhi Wang, Yongqiang Dou, Yanjun Chen, Li~Fei-Fei, Anima Anandkumar, Yuke Zhu, and Linxi Fan.
\newblock Vima: General robot manipulation with multimodal prompts.
\newblock In {\em NeurIPS 2022 Foundation Models for Decision Making Workshop}, 2022.

\bibitem{gur2023real}
Izzeddin Gur, Hiroki Furuta, Austin Huang, Mustafa Safdari, Yutaka Matsuo, Douglas Eck, and Aleksandra Faust.
\newblock A real-world webagent with planning, long context understanding, and program synthesis.
\newblock {\em arXiv preprint arXiv:2307.12856}, 2023.

\bibitem{shaw2024pixels}
Peter Shaw, Mandar Joshi, James Cohan, Jonathan Berant, Panupong Pasupat, Hexiang Hu, Urvashi Khandelwal, Kenton Lee, and Kristina~N Toutanova.
\newblock From pixels to ui actions: Learning to follow instructions via graphical user interfaces.
\newblock {\em Advances in Neural Information Processing Systems}, 36, 2024.

\bibitem{zhou2023webarena}
Shuyan Zhou, Frank~F Xu, Hao Zhu, Xuhui Zhou, Robert Lo, Abishek Sridhar, Xianyi Cheng, Yonatan Bisk, Daniel Fried, Uri Alon, et~al.
\newblock Webarena: A realistic web environment for building autonomous agents.
\newblock {\em arXiv preprint arXiv:2307.13854}, 2023.

\bibitem{yao2024tree}
Shunyu Yao, Dian Yu, Jeffrey Zhao, Izhak Shafran, Tom Griffiths, Yuan Cao, and Karthik Narasimhan.
\newblock Tree of thoughts: Deliberate problem solving with large language models.
\newblock {\em Advances in Neural Information Processing Systems}, 36, 2024.

\bibitem{long2023large}
Jieyi Long.
\newblock Large language model guided tree-of-thought.
\newblock {\em arXiv preprint arXiv:2305.08291}, 2023.

\bibitem{shinn2023reflexion}
Noah Shinn, Beck Labash, and Ashwin Gopinath.
\newblock Reflexion: an autonomous agent with dynamic memory and self-reflection.
\newblock {\em arXiv preprint arXiv:2303.11366}, 2023.

\bibitem{kim2024language}
Geunwoo Kim, Pierre Baldi, and Stephen McAleer.
\newblock Language models can solve computer tasks.
\newblock {\em Advances in Neural Information Processing Systems}, 36, 2024.

\bibitem{sridhar2023hierarchical}
Abishek Sridhar, Robert Lo, Frank~F Xu, Hao Zhu, and Shuyan Zhou.
\newblock Hierarchical prompting assists large language model on web navigation.
\newblock In {\em The 2023 Conference on Empirical Methods in Natural Language Processing}, 2023.

\bibitem{carta2023grounding}
Thomas Carta, Cl{\'e}ment Romac, Thomas Wolf, Sylvain Lamprier, Olivier Sigaud, and Pierre-Yves Oudeyer.
\newblock Grounding large language models in interactive environments with online reinforcement learning.
\newblock In {\em International Conference on Machine Learning}, pages 3676--3713. PMLR, 2023.

\bibitem{kwon2023reward}
Minae Kwon, Sang~Michael Xie, Kalesha Bullard, and Dorsa Sadigh.
\newblock Reward design with language models.
\newblock {\em arXiv preprint arXiv:2303.00001}, 2023.

\bibitem{wu2024read}
Yue Wu, Yewen Fan, Paul~Pu Liang, Amos Azaria, Yuanzhi Li, and Tom~M Mitchell.
\newblock Read and reap the rewards: Learning to play atari with the help of instruction manuals.
\newblock {\em Advances in Neural Information Processing Systems}, 36, 2024.

\bibitem{chu2023accelerating}
Kun Chu, Xufeng Zhao, Cornelius Weber, Mengdi Li, and Stefan Wermter.
\newblock Accelerating reinforcement learning of robotic manipulations via feedback from large language models.
\newblock {\em arXiv preprint arXiv:2311.02379}, 2023.

\bibitem{yu2023language}
Wenhao Yu, Nimrod Gileadi, Chuyuan Fu, Sean Kirmani, Kuang-Huei Lee, Montserrat~Gonzalez Arenas, Hao-Tien~Lewis Chiang, Tom Erez, Leonard Hasenclever, Jan Humplik, et~al.
\newblock Language to rewards for robotic skill synthesis.
\newblock In {\em 7th Annual Conference on Robot Learning}, 2023.

\bibitem{ma2023eureka}
Yecheng~Jason Ma, William Liang, Guanzhi Wang, De-An Huang, Osbert Bastani, Dinesh Jayaraman, Yuke Zhu, Linxi Fan, and Anima Anandkumar.
\newblock Eureka: Human-level reward design via coding large language models.
\newblock In {\em The Twelfth International Conference on Learning Representations}, 2023.

\bibitem{paischer2022history}
Fabian Paischer, Thomas Adler, Vihang Patil, Angela Bitto-Nemling, Markus Holzleitner, Sebastian Lehner, Hamid Eghbal-Zadeh, and Sepp Hochreiter.
\newblock History compression via language models in reinforcement learning.
\newblock In {\em International Conference on Machine Learning}, pages 17156--17185. PMLR, 2022.

\bibitem{paischer2024semantic}
Fabian Paischer, Thomas Adler, Markus Hofmarcher, and Sepp Hochreiter.
\newblock Semantic helm: A human-readable memory for reinforcement learning.
\newblock {\em Advances in Neural Information Processing Systems}, 36, 2024.

\bibitem{radford2021learning}
Alec Radford, Jong~Wook Kim, Chris Hallacy, Aditya Ramesh, Gabriel Goh, Sandhini Agarwal, Girish Sastry, Amanda Askell, Pamela Mishkin, Jack Clark, et~al.
\newblock Learning transferable visual models from natural language supervision.
\newblock In {\em International conference on machine learning}, pages 8748--8763. PMLR, 2021.

\bibitem{chen2021decision}
Lili Chen, Kevin Lu, Aravind Rajeswaran, Kimin Lee, Aditya Grover, Misha Laskin, Pieter Abbeel, Aravind Srinivas, and Igor Mordatch.
\newblock Decision transformer: Reinforcement learning via sequence modeling.
\newblock {\em Advances in neural information processing systems}, 34:15084--15097, 2021.

\bibitem{micheli2022transformers}
Vincent Micheli, Eloi Alonso, and Fran{\c{c}}ois Fleuret.
\newblock Transformers are sample-efficient world models.
\newblock {\em arXiv preprint arXiv:2209.00588}, 2022.

\bibitem{robine2023transformer}
Jan Robine, Marc H{\"o}ftmann, Tobias Uelwer, and Stefan Harmeling.
\newblock Transformer-based world models are happy with 100k interactions.
\newblock {\em arXiv preprint arXiv:2303.07109}, 2023.

\bibitem{chen2022transdreamer}
Chang Chen, Yi-Fu Wu, Jaesik Yoon, and Sungjin Ahn.
\newblock Transdreamer: Reinforcement learning with transformer world models.
\newblock {\em arXiv preprint arXiv:2202.09481}, 2022.

\bibitem{geng2023improving}
Sinong Geng, Aldo Pacchiano, Andrey Kolobov, and Ching-An Cheng.
\newblock Improving offline rl by blending heuristics.
\newblock {\em arXiv preprint arXiv:2306.00321}, 2023.

\bibitem{lin2018episodic}
Zichuan Lin, Tianqi Zhao, Guangwen Yang, and Lintao Zhang.
\newblock Episodic memory deep q-networks.
\newblock In {\em Proceedings of the 27th International Joint Conference on Artificial Intelligence}, pages 2433--2439, 2018.

\bibitem{kocsis2006bandit_mcts}
Levente Kocsis and Csaba Szepesv{\'a}ri.
\newblock Bandit based monte-carlo planning.
\newblock In {\em European conference on machine learning}, pages 282--293. Springer, 2006.

\bibitem{browne2012survey_mcts}
Cameron~B Browne, Edward Powley, Daniel Whitehouse, Simon~M Lucas, Peter~I Cowling, Philipp Rohlfshagen, Stephen Tavener, Diego Perez, Spyridon Samothrakis, and Simon Colton.
\newblock A survey of monte carlo tree search methods.
\newblock {\em IEEE Transactions on Computational Intelligence and AI in games}, 4(1):1--43, 2012.

\bibitem{coulom2006efficient_mcts}
R{\'e}mi Coulom.
\newblock Efficient selectivity and backup operators in monte-carlo tree search.
\newblock In {\em International conference on computers and games}, pages 72--83. Springer, 2006.

\bibitem{rashidinejad2021bridging}
Paria Rashidinejad, Banghua Zhu, Cong Ma, Jiantao Jiao, and Stuart Russell.
\newblock Bridging offline reinforcement learning and imitation learning: A tale of pessimism.
\newblock {\em Advances in Neural Information Processing Systems}, 34:11702--11716, 2021.

\bibitem{hu2023language}
Hengyuan Hu and Dorsa Sadigh.
\newblock Language instructed reinforcement learning for human-ai coordination.
\newblock In {\em International Conference on Machine Learning}, pages 13584--13598. PMLR, 2023.

\bibitem{kumar2020conservative_cql}
Aviral Kumar, Aurick Zhou, George Tucker, and Sergey Levine.
\newblock Conservative q-learning for offline reinforcement learning.
\newblock {\em Advances in Neural Information Processing Systems}, 33:1179--1191, 2020.

\bibitem{NIPS2010_091d584f}
Hado Hasselt.
\newblock Double q-learning.
\newblock In J.~Lafferty, C.~Williams, J.~Shawe-Taylor, R.~Zemel, and A.~Culotta, editors, {\em Advances in Neural Information Processing Systems}, volume~23. Curran Associates, Inc., 2010.

\bibitem{schulman2017proximal_ppo}
John Schulman, Filip Wolski, Prafulla Dhariwal, Alec Radford, and Oleg Klimov.
\newblock Proximal policy optimization algorithms.
\newblock {\em arXiv preprint arXiv:1707.06347}, 2017.

\bibitem{huang2022cleanrl}
Shengyi Huang, Rousslan Fernand~Julien Dossa, Chang Ye, Jeff Braga, Dipam Chakraborty, Kinal Mehta, and João~G.M. Araújo.
\newblock Cleanrl: High-quality single-file implementations of deep reinforcement learning algorithms.
\newblock {\em Journal of Machine Learning Research}, 23(274):1--18, 2022.

\bibitem{lillicrap2015continuous_ddpg}
Timothy~P Lillicrap, Jonathan~J Hunt, Alexander Pritzel, Nicolas Heess, Tom Erez, Yuval Tassa, David Silver, and Daan Wierstra.
\newblock Continuous control with deep reinforcement learning.
\newblock {\em arXiv preprint arXiv:1509.02971}, 2015.

\bibitem{abreu2023designing}
Miguel Abreu, Luis~Paulo Reis, and Nuno Lau.
\newblock Designing a skilled soccer team for robocup: Exploring skill-set-primitives through reinforcement learning.
\newblock {\em arXiv preprint arXiv:2312.14360}, 2023.

\bibitem{deepmind_soccer_sr}
Tuomas Haarnoja, Ben Moran, Guy Lever, Sandy~H Huang, Dhruva Tirumala, Jan Humplik, Markus Wulfmeier, Saran Tunyasuvunakool, Noah~Y Siegel, Roland Hafner, et~al.
\newblock Learning agile soccer skills for a bipedal robot with deep reinforcement learning.
\newblock {\em arXiv preprint arXiv:2304.13653}, 2023.

\end{thebibliography}
\appendix


\section{Notation}

Vectors are denoted by bold lowercase letters, $\mathbf{a}$, and matrices by uppercase letters, $A$. Individual vector elements or matrix rows are referenced using function notation, $\mathbf{a}(x)$. The identity matrix is represented as $\mathit{I}$. 
We use $\mathbb{E}_{p}[\cdot]$ to denote the expected value of a function under a distribution $p$. Specifically, for any space $\mathcal{X}$, distribution $p \in \operatorname{dist}(\mathcal{X})$, and function $\mathbf{a}:\mathcal{X}\to\real$, we have $\mathbb{E}_{p}[\mathbf{a}] := \mathbb{E}_{x \sim p} [\mathbf{a}(x)]$. 
The notation $|\cdot|$ indicates the element-wise absolute value of a vector, such that $|\mathbf{a}|(x) = |\mathbf{a}(x)|$. The infinity norm, $|\mathbf{a} |_{\infty}$, represents the maximum absolute value among the elements of $\mathbf{a}$.

We define \( G: \mathbb{R}^n \to \mathbb{R}^n \) as the LLM Generator, where \( n \) denotes the context length. The prompt \( p_0 \) serves as a guiding template to direct the LLM's output towards generating Python code. In our implementation, the input to \( G \) is the concatenation of \( p_0 \) and the task description, resulting in an executable Python method that outputs tuples \( (s, a, Q) \).
To represent the transformation from text to function or dataset. We denote $\mathcal{T}(\cdot)$ as the transformation from text to function, and use $\mathcal{R}(\cdot)$ to denote the transformation from text to reward.

\paragraph{Markov Decision Processes.}
We represent the environment with which we are interacting as a Markov Decision Process (MDP), defined in standard fashion: $\mdp := \langle\sspace, \aspace, \rew, \trns, \gamma, \rho\rangle$. 
$\sspace$ and $\aspace$ denote the state and action space, which we assume are discrete. 
We use $\saspace:=\sspace\times\aspace$ as the shorthand for the joint state-action space. The reward function $\rew \colon \saspace \to \dist([0,1])$ maps state-action pairs to distributions over the unit interval, while the transition function $\trns \colon \saspace \to \dist(\sspace)$ maps state-action pairs to distributions over next states. 
Finally, $\rho \in \dist(\sspace)$ is the distribution over initial states. 
We use $\erew$ to denote the expected reward function, $\erew(\sa) := \E_{r \sim \rew(\cdot \mid \sa)}[r]$, which can also be interpreted as a vector $\erew\in\real^{|\saspace|}$. 
Similarly, note that $\trns$ can be described as $\trns \colon (\saspace \times \sspace) \to \real$, which can be represented as a stochastic matrix $\trns \in \real^{|\saspace|\times|\sspace|}$.
We denote $\erewm$ and $\trnsm$ as the true environment's reward and transition functions.
For policy definition, we denote the space of all possible policies as $\pispace$. A policy \(\pi: \mathcal{S} \rightarrow \Delta(\mathcal{A})\) defines a conditional distribution over actions given states. A deterministic policy \(\mu: \mathcal{S} \rightarrow \mathcal{A}\) is a special case of \(\pi\), directly selecting one action per state with a probability of 1 for that action and 0 for others.
 We define an ``activity matrix'' $\actpi\in\real^{\sspace\times\saspace}$ for each policy, encoding $\pi$'s state-conditional state-action distribution. Specifically,  $\actpi(s, \langle \dot{s},a \rangle):=\pi(a|s)$ if $s = \dot{s}$, otherwise $\actpi(s, \langle \dot{s},a \rangle):=0$.
 We define a value function as $v \colon \Pi\to\sspace\to\real$ or $q \colon \Pi\to\saspace\to\real$, with bounded outputs.  
 We denote the value functions of a specific policy as $\valpi := \val(\pi)$ and $\qvalpi := \qval(\pi)$, representing them as vectors: $\valpi \in \real^{|\sspace|}$ and $\qvalpi \in \real^{|\saspace|}$. 
 We use $\qval^*$ as the abbreviation of $ \bigl(\qval_{\mdp}^{\pi_\mdp^*} \bigr)^*$,  $\qval_\xi^*$ as the abbreviation of $\bigl(\qval_{\xi}^{\pi_\xi^*} \bigr)^*$.
 The terms $\qval$ and $\val$ refer to discrete matrix representations, and specifically $\val(s)$ and $\qval(s)$ denote the output of an arbitrary value function on an arbitrary policy on a state,  while $Q : \real^{|\saspace|} \to \real$ denotes a more general definition suitable for both discrete and continuous contexts. 

The \textit{expected return} of an MDP $\mdp$, denoted $\rho_{\mdp}^T\val^*_\mdp$ or $\rho_{\mdp}^T\qval^*_\mdp$, is the discounted sum of rewards starting from initial states, acquired when interacting with the environment:
\begin{align*}
    \rho_{\mdp}^T\val^*_\mdp(\pi) &:= \rho^T\sum_{t=0}^\infty \left( \gamma \actpi \trns \right)^t \actpi \erew &
    \rho_{\mdp}^T\qval^*_\mdp(\pi) &:= \rho^T\sum_{t=0}^\infty \left( \gamma \trns \actpi \right)^t \erew
\end{align*}
Note that $(\valpi_\mdp)^* = \actpi (\qvalpi_\mdp)^*$. 
An \textit{optimal policy} of an MDP, denoted $\optp_\mdp$, maximizes the expected return under the initial state distribution: $\optp_\mdp := \arg\max_{\pi} \E_{\rho} [\valpi_\mdp]$.
The state-wise expected returns of an optimal policy can be written as $\val_\mdp^{\optp_\mdp}$. 
The Bellman consistency equation for $\mathbf{x}$ is  $\bell_\mdp(\mathbf{x}) := \erew + \gamma \trns \mathbf{x}$. Since $(\valpi_\mdp)^*$ is the only vector for which $(\valpi_\mdp)^* = \actpi \bell_\mdp((\valpi_\mdp)^*)$ holds. 
For any state $s$, the probability of reaching state $s'$ after t time steps under policy $\pi$ is $[(\actpi\trns)^t](s, s')$. Furthermore, $\sum_{t=0}^\infty \left( \gamma \actpi\trns \right)^t = \left( \Id -  \gamma \actpi \trns \right)^{-1}$. The discounted visitation of $\pi$ is $ \frac{1}{1-\gamma}\left( \Id -  \gamma \actpi \trns \right)^{-1}$.

\paragraph{Datasets.} 

We define fundamental concepts crucial for fixed-dataset policy optimization. Let $D := {\langle s,a,r,s'\rangle }^d$ represent a dataset of $d$ transitions, and let $\dataspace$ denote the set of all such datasets. 
Our focus is on datasets sampled from a distribution $\Phi : \dist(\saspace)$, typically based on state-action pairs generated by a stationary policy. 
A dataset $D$ containing $d$ tuples $\langle s,a,r,s'\rangle$ is sampled as $D \sim \Phi_d$, where pairs $\sa$ are drawn from $\Phi$, and rewards $r$ and subsequent states $s'$ are sampled independently from the reward function $\rew(\cdot|\sa)$ and the transition function $\trns(\cdot|\sa)$, respectively.
We denote $D(\sa)$ as the multi-set of all $\langle r,s' \rangle$ pairs, and use $\countdsa \in \real^{|\saspace|}$ to denote the count vector, where $\countdsa(\sa) := |D(s,a)|$. 
We define the empirical reward vector as $\erewd(\sa) := \sum_{r,s' \in D(\sa)} \frac{r}{|D(\sa)|}$ and empirical transition matrix as $\trnsd(s'|\sa) :=  \sum_{r,\dot{s}' \in D(\sa)} \frac{\mathbb{I}(\dot{s}' = s')}{|D(\sa)|}$ for all state-action pairs with $\countdsa(\sa) > 0$. 
For state-action pairs where $\countdsa(\sa) =0$, the maximum-likelihood estimates of reward and transition cannot be clearly defined, so they remain unspecified.
The bounds hold no matter how these values are chosen, so long as $\erewd$ is bounded and $\trnsd$ is stochastic. 
The empirical policy of a dataset $D$ is defined as $\emp(a|s) := \frac{|D(\sa)|}{|D(\langle s,\cdot \rangle)|}$ except where $\countdsa(\sa) = 0$, where it can similarly be any valid action distribution. 
The empirical visitation distribution of a dataset $D$ is computed analogously to the regular visitation distribution but uses $\trnsd$ in place of $\trns$. Thus it's given by $ \frac{1}{1-\gamma}\left( \Id -  \gamma \actpi \trnsd \right)^{-1}$.

\section{Proofs}

\begin{lemma}[Decomposition]
\label{lemma:delta_v_transform}
For any MDP $\xi$ and policy $\pi$, consider the Bellman fixed-point equation given by, let $(\valpi_\xi)^*$ be defined as the unique value vector such that $(\valpi_\xi)^* = \actpi(\rawerew_\xi + \gamma \rawtrns_\xi (\valpi_\xi)^*)$, and let $\val$ be any other value vector. Assume that $\pi(a|s) = 1$ if $a=\arg \max_a (\qval_\xi^{\pi})^*(s,a)$, otherwise $\pi(a|s) = 0$.
We have:
\begin{align}
|\qval_\xi^*(s,\mu(s)) -  \qval(s,\mu(s))| &= |\bigl((\Id -  \gamma \actpi \trns_\xi)^{-1} (\actpi(\rawerew_\xi + \gamma \rawtrns_\xi \val) - \val)\bigr)(s)|
\end{align}
\end{lemma}

\begin{proof}
\begin{align*}
\actpi(\rawerew_\xi + \gamma \rawtrns_\xi \val) - \val &= \actpi(\rawerew_\xi + \gamma \rawtrns_\xi \val) - (\valpi_\xi)^* + (\valpi_\xi)^* - \val \\
&= \actpi(\erew_\xi + \gamma \trns_\xi \val) - \actpi(\erew_\xi + \gamma \trns_\xi (\valpi_\xi)^*) + (\valpi_\xi)^* -  \val \\
&= \gamma \actpi \trns_\xi (\val - (\valpi_\xi)^*) + ((\valpi_\xi)^* -  \val) \\
&= ((\valpi_\xi)^* -  \val) - \gamma \actpi \trns_\xi ((\valpi_\xi)^* -  \val) \\
&= (\Id -  \gamma \actpi \trns_\xi)((\valpi_\xi)^* -  \val)
\end{align*}
Note that $(\val_\xi^{\pi})^* = \actpi(\qval_\xi^\pi)^*$, After we expand the value function we have:
\begin{align*}
(\Id -  \gamma \actpi \trns_\xi)^{-1}\bigl(\actpi(\rawerew_\xi + \gamma \rawtrns_\xi \val) \bigr) &= \actpi(\qval_{\xi}^\pi)^* - \val \\
&= \actpi(\qval_{\xi}^\pi)^* - \act\qval
\end{align*}
By indexing at $\langle s,\mu(s) \rangle$, we have:
$$|\qval_\xi^*(s,\mu(s)) -  \qval(s,\mu(s))| = |\bigl((\Id -  \gamma \actpi \trns_\xi)^{-1} (\actpi(\rawerew_\xi + \gamma \rawtrns_\xi \val) - \val)\bigr)(s)|$$

\end{proof}

\begin{lemma}[Convergence Bound]
\label{lemma:q_value_bound}   

Suppose that $s'$ and $r$ are sampled independently and identically distributed (iid) from $P(\cdot | s,a)$ and $R(\cdot | s,a)$ respectively. Then, with probability at least $1 - \delta$, we have:

$$| \qval_D^*(s,\mu(s)) - \hat{\qval}(s,\mu(s))| \le \left(\sqrt{\frac{1}{2} \ln \frac{2|\sspace\times\aspace|}{\delta}}\right) \sum_{s'}\nu_D(s'|s_0=s)\frac{1}{\sqrt{\countdsa(\langle s',\mu(s')\rangle)} } $$
\end{lemma}
\begin{proof}
See proof at \ref{sec:proof_q_value_bound}
\end{proof}

\subsection{Proof of Theorem \ref{thm:qhat_is_contraction}}
\label{sec:proof_qhat_contraction}
Note that $\qval$ is the matrix representation of the Q function. In the proof of this section, we use a more general $Q : \real^\saspace \to \real$ to represent the Q function.
The update formula for $\hat{Q}$ iteration is:
$$\hat{Q}^{k+1}(s,a) =(1-\alpha) \hat{Q}^{k}(s,a) 
+ \alpha \left(  r(s,a) + \gamma\sum_{s' \in S}P(s'|s,a)\max_a \hat{Q}(s',a) +\mathbf{h}(s,a)\right)$$
And the update formula can be defined as \textbf{Bellman optimal operator $\mathcal{B}$}:
$$\hat{Q}^*(s,a)  = \mathcal{B}\hat{Q}^* = r(s,a) +\gamma \sum_{s' \in S}P(s'|s,a) \max_a \hat{Q}^*(s',a)$$
Next we prove that the Bellman optimal operator $\mathcal{B}$ is $\gamma$-contraction operator on $\hat{Q}$:
\begin{align*}
    \|\mathcal{B}\hat{Q} - \mathcal{B}\hat{Q}'\|_{\infty} &=
    \gamma \max_{s,a \in \mathcal{S,A}}|\sum_{s'}P(s'|s,a)[\max_a \hat{Q}(s',a) - \max_a \hat{Q}'(s',a)]| \\
    &\le  \gamma \max_{s,a \in \mathcal{S,A}}|\max_{s'}|\left(\max_a \hat{Q}(s',a) - \max_a \hat{Q}'(s',a)\right)|| \\
    & = \gamma \| \hat{Q}-\hat{Q}'\|_{\infty}
\end{align*}
The Bellman optimal operator $\mathcal{B}$ is a $\gamma$-contraction operator on $\hat{Q}$. This means that as the number of iteration k increases and we stop providing heuristic values before convergence, $\hat{Q}^k$ will get closer to the fixed point, i.e., $\hat{Q}^*$.
Next, we prove that the converged heuristic-guided Q function equals the traditional Q function. Define the following:
\begin{align*}
    \Theta_H & \text{ denotes the set of terminal states,} \\
    \Theta_{H-1} & \text{ denotes the set of states one step before the terminal,} \\
    & \vdots \\
    \Theta_1 & \text{ denotes the set of states at the initial step.}
\end{align*}
For $s \in \Theta_{H-1}$, it's clear that $\hat{Q}^*(s) = Q^*(s)$, because:
$$ Q^*(s,a)|_{s \in \Theta_{H-1}}=\hat{Q}^*(s,a)|_{s \in \Theta_{H-1}} = r(s,a)+\gamma Q^*(s')=r(s,a)$$
For $s \in \Theta_{H-2}$, which is the previous state in $\Theta_{H-1}$, we have:
\begin{align*}
\hat{Q}^*(s,a)|_{s \in \Theta_{H-2}} &= r(s,a)+\gamma \sum_{s' \in \Theta_{H-1} } P(s'|s,a) \hat{Q}^*(s')\\
&= r(s,a)+\gamma \sum_{s' \in \Theta_{H-1} } P(s'|s,a) Q^*(s')\\
&=  Q^*(s,a)|_{s \in \Theta_{H-2}}
\end{align*}
With constant iteration, we have: $\hat{Q}^* = Q^*$ . Specifically, we have: $\qval_{D} = \hat{\qval}_{D}$

\subsection{Proof of Suboptimality}
\label{sec:proof_suboptimality}
\begin{align*}
(\qval^{\optp})^*(s,a^{\optp}) - (\qval^{\optpd})^*(s,a^{\optpd}) &= 
(\qval^{\optp})^*(s,a^*) +[- (\hat{\qval}_D^{\pi})^*(s,a^{\pi}) + (\hat{\qval}_D^{\pi})^*(s,a^{\pi}) ] +\\&\hspace{4cm}[- (\hat{\qval}_D^{\optpd})^*(s,a^{\optpd}) +(\hat{\qval}_D^{\optpd})^*(s,a^{\optpd})]  - (\qval^{\optpd})^*(s,a^{\optpd})
\\&= [(\qval^{\optp})^*(s,a^*) - (\hat{\qval}_D^{\pi})^*(s,a^{\pi})] + [(\hat{\qval}_D^{\pi})^*(s,a^{\pi}) - (\hat{\qval}_D^{\optpd})^*(s,a^{\optpd})] 
\\&\hspace{6cm}+[(\hat{\qval}_D^{\optpd})^*(s,a^{\optpd}) - (\qval^{\optpd})^*(s,a^{\optpd})]
\\ &\le [(\qval^{\optp})^*(s,a^*) - (\hat{\qval}_D^{\pi})^*(s,a^{\pi})]  +[(\hat{\qval}_D^{\optpd})^*(s,a^{\optpd}) - (\qval^{\optpd})^*(s,a^{\optpd})] 
\end{align*}
We treat the RHS as a function of $\pi$, so we have:
\begin{align*}
    (\qval^{\optp})^*(s,a^{\optp}) - (\qval^{\optpd})^*(s,a^{\optpd}) &\le
     \inf_{\pi \in \pispace} \bigl(  [(\qval^{\optp})^*(s,a^*) - (\hat{\qval}_D^{\pi})^*(s,a^{\pi})]  +  [(\hat{\qval}_D^{\optpd})^*(s,a^{\optpd}) - (\qval^{\optpd})^*(s,a^{\optpd})] \bigr)
     \\&\le  \inf_{\pi \in \pispace} \bigl(  (\qval^{\optp})^*(s,a^*) - (\hat{\qval}_D^{\pi})^*(s,a^{\pi}) \bigr) +  \sup_{\pi \in \pispace}((\hat{\qval}_D^{\pi})^*(s,a^\pi) - (\qval^{\pi})^*(s,a^\pi))
     \\&=\inf_{\pi \in \pispace} \bigl(  (\qval^{\optp})^*(s,a^*) - (\qval_D^{\optpd})^*(s,a^{\optp}) + (\qval_D^{\optpd})^*(s,a^{\optp}) +(\hat{\qval}_D^{\pi})^*(s,a^{\pi}) \bigr)  \\&\hspace{6cm} +\sup_{\pi \in \pispace}((\hat{\qval}_D^{\pi})^*(s,a^\pi) - (\qval^{\pi})^*(s,a^\pi)) 
\end{align*}

\subsection{Proof of Lemma \ref{lemma:q_value_bound}}
\label{sec:proof_q_value_bound}
Consider that for any $\sa$, the expression $\erewd(\sa) + \gamma \trnsd(\sa) \valpi$ can be equivalently expressed as an expectation of random variables, $$\erewd(\sa) + \gamma \trnsd(\sa) \val = \frac{1}{\countdsa(\sa)}\sum_{r,s' \in D(\sa)} r + \gamma \val(s')$$
each with expected value:
$$\E_{r,s' \in D(\sa)} [r + \gamma \val(s')] = \E_{\substack{r \sim \rew(\cdot|\sa) \\ s' \sim \trns(\cdot|\sa)}} [r + \gamma \val(s')] = [\erewm + \gamma \trnsm \val](\sa).$$
 
 Hoeffding's inequality indicates that the mean of bounded random variables will approximate their expected values with high probability. By applying Hoeffding's inequality to each element in the $|\sspace \times \aspace|$ state-action space and employing a union bound, we establish that with probability at least $1 - \delta$,
$$\left|(\erewm + \gamma \trnsm \val) - \left(\erewd + \gamma \trnsd \val \right)\right| \leq \frac{1}{1-\gamma}\sqrt{\frac{1}{2} \ln \frac{2|\sspace\times\aspace|}{\delta}\countdsa^{-1}}$$
We can left-multiply $\actpi$ and rearrange to get:
$$\left|\actpi(\erewm + \gamma \trnsm \val) - \actpi\left(\erewd + \gamma \trnsd \val \right)\right| \leq \left(\frac{1}{1-\gamma}\sqrt{\frac{1}{2} \ln \frac{2|\sspace\times\aspace|}{\delta}}\right)\actpi\ircountdsa$$
then we left-multiply the discounted visitation of $\pi$:
\begin{align*}
|\left( \Id -  \gamma \actpi \trnsd \right)^{-1}[\actpi(\erewm + \gamma \trnsm \val) - \actpi\left(\erewd + \gamma \trnsd \val \right)]| &\leq
\left(\frac{1}{1-\gamma}\sqrt{\frac{1}{2} \ln \frac{2|\sspace\times\aspace|}{\delta}}\right) \left( \Id -  \gamma \actpi \trnsd \right)^{-1} \actpi\ircountdsa \\
\end{align*}

This matrix: $\left( \Id -  \gamma \actpi \trnsd \right)^{-1} \actpi\ircountdsa$,belongs to the space $\mathbb{R}^{|S|}$. By indexing at state $s$, we obtain:

$$\left( \Id -  \gamma \actpi \trnsd \right)^{-1} \actpi\ircountdsa (s) = (1-\gamma)\sum_{s'}\nu(s'|s_0=s)\frac{1}{\sqrt{N_D(\smu)} }$$

Finally, by integrate these terms together we have the bound on Lemma \ref{lemma:q_value_bound}:

$$| \qval_D^*(s,\mu(s)) - \qval(s,\mu(s))| \le \left(\sqrt{\frac{1}{2} \ln \frac{2|\sspace\times\aspace|}{\delta}}\right) \sum_{s'}\nu(s'|s_0=s)\frac{1}{\sqrt{N_D(\langle s',\mu(s')\rangle)} } $$

Given that this inequality is universally applicable to any $\qval$, and acknowledging that the heuristic term $\mathbf{h}$ supplied by the LLM serves as a constant within the temporal-difference (TD) update mechanism of the Q-function, it follows that:

\begin{align*}
    | \qval_D^*(s,\mu(s)) - \hat{\qval}(s,\mu(s))| &=  | \qval_D^*(s,\mu(s)) - \qval(s,\mu(s))-\mathbf{h}(s,\mu(s))|\\
    &\le \left(\sqrt{\frac{1}{2} \ln \frac{2|\sspace\times\aspace|}{\delta}}\right) \sum_{s'}\nu(s'|s_0=s)\frac{1}{\sqrt{N_D(\langle s',\mu(s')\rangle)} }
\end{align*}

\subsection{Proof of Theorem \ref{thm:sample complexity}}
\label{sec:proof_sample_complexity}

To get the sample complexity of convergence. By Lemma \ref{lemma:q_value_bound},we have:
\begin{align*}
        | \qval_D^*(s,\mu(s)) - \hat{\qval}(s,\mu(s))| 
        &\le \left(\sqrt{\frac{1}{2} \ln \frac{2|\sspace\times\aspace|}{\delta}}\right) \sum_{s'}\nu(s'|s_0=s)\frac{1}{\sqrt{N_D(\langle s',\mu(s')\rangle)} } \\
        & =  \left(\sqrt{\frac{1}{2} \ln \frac{2|\sspace\times\aspace|}{\delta}}\right) \sum_{s'}\sqrt{\nu(s'|s_0=s)}  \frac{\sqrt{\nu(s'|s_0=s)}}
        {\sqrt{nd_D(s,a)} } \tag{ $d_D(s,a) = \frac{N_D(\langle s,a \rangle)}{|D|}$} \\
        &=\left(\sqrt{\frac{1}{2} \ln \frac{2|\sspace\times\aspace|}{\delta}}\right) \sum_{s'}\sqrt{d_D(s,\mu(s))}  \frac{\sqrt{d_D(s,\mu(s))}}
        {\sqrt{nd_D(s,a)} } \tag{ $\nu(s)\pi(\mu(s)|s) \approx d_D(s,\mu(s))
        $} \\
        &\le \frac{\left(\sqrt{\frac{1}{2} \ln \frac{2|\sspace\times\aspace|}{\delta}}\right)}{\sqrt{n}} \sum_{s'} \sqrt{d_D(s',\mu(s))}\\
        & \le  \left(\sqrt{\frac{1}{2} \ln \frac{2|\sspace\times\aspace|}{\delta}}\right)\frac{ |S|}{\sqrt{n}}
\end{align*}
Then, when $n > \mathcal{O}\left( \frac{|S|^2}{2\epsilon^2}\ln{\frac{2|S \times A|}{\delta}} \right)$,we have $| \qval_D^*(s,\mu(s)) - \qval^*(s,\mu(s))| \le \epsilon$.

\section{Online-Guidance Implementation}
\label{sec:online_guidance}
In the implementation of the online-guidance Q-learning algorithm, each training step involves checking for an input of the $(s, a, Q)$ pair. Upon detecting this input, the algorithm utilizes Equation \ref{eq:online_hq} to train the $(s, a, Q)$ pair in conjunction with the target item. This process adjusts the $Q$-function, thereby influencing the agent's decision-making behavior in subsequent sampling steps.

\begin{algorithm}[H]
\caption{Online Heuristic Q Learning Algorithm}
\label{alg:online_hq}
\begin{algorithmic}[1]
\State \textbf{Inputs:} Local MDP $D$, Large Language Model $G$,prompt $p$,
\State \textbf{Initialization:} Initialize Heuristic:$G(p)$,initialize actor-critic $(\mu_\phi,Q_{\theta_1},Q_{\theta_2})$ and target actor-critic $(\mu_{\phi^{'}},Q_{\theta^{'}_1},Q_{\theta^{'}_2})$ 
\State \textbf{Generate Q buffer:} $D_g \leftarrow \{ (s_i, a_i,Q_i) \mid (s_i, a_i,Q_i) = D(G(p)), \, i = 1, 2, \ldots, n \}$
\State \textbf{Q Boostrapping:} $\theta =\theta -\alpha \nabla_\theta L_{bootstrap}$ using  eq \ref{eq:bootstrap_loss}
    \For{iteration $t' \in T=1,2,3...$}
        \State Sample $(s,a,r,s')$ from Env
        \State $D \leftarrow D \cup (s,a,r,s')$        
        \State Sample N transitions $(s,a,r,s')$ from $D$
        \State \textbf{Detect interaction:} $D_g = \begin{cases}D(G(p)) & \text{if detected} \\\text{null} & \text{if not}  \end{cases}$ 
        \State $\hat{a} \leftarrow \mu(s') + \epsilon, \epsilon \sim \text{clip}(\mathcal{N}(0,\sigma),-c,c)$
        \State $y(s') \leftarrow \lceil r+\gamma \min_{i=1,2}\qval_{\theta_i}(s',\hat{a})  \rceil $
        \State Update critics $\theta_i  \leftarrow \arg \min_{\theta_i}  L_{online}(\theta_i)$
        
        \If{ $t' \mod d$}
            \State Update actor $\phi \leftarrow \arg \min_\phi L_{actor}(\phi)$ 
            \State Update target networks:
            \State $\theta_{i}^{'} \leftarrow \tau\theta_i + (1-\tau)\theta_{i}^{'}$
            \State $\phi^{'} \leftarrow \tau\phi + (1-\tau)\phi^{'} $
        \EndIf
        \EndFor
\State \textbf{end while}
\end{algorithmic}
\end{algorithm}

\section{Experiment Details}
Here we use table \ref{tab:code_repo} to list the open source repositories of the algorithms used in the experiment.

\begin{table}[h]
  \caption{Baseline Code Source}
  \label{tab:code_repo}
  \centering
  \begin{tabular}{ll}
    \toprule
    Algorithm     & Code Repository     \\
    \midrule
    PPO & \href{https://docs.cleanrl.dev/rl-algorithms/ppo/}{https://docs.cleanrl.dev/rl-algorithms/ppo/}     \\
    TD3     & \href{https://github.com/sfujim/TD3}{https://github.com/sfujim/TD3}      \\
    DDPG     & \href{https://github.com/sfujim/TD3}{https://github.com/sfujim/TD3 }       \\
    SAC &   \href{https://github.com/boyu-ai/Hands-on-RL}{https://github.com/boyu-ai/Hands-on-RL} \\ &\href{https://github.com/haarnoja/sac}{https://github.com/haarnoja/sac}
    
    \\ 
    \bottomrule
  \end{tabular}
\end{table}

In our experiments, we utilized "ChatGPT-Classic" as the language model to provide heuristic Q-values, thereby accelerating the exploration process in the LLM-TD3 algorithm. The experiments were conducted on a host equipped with a 48-core CPU, 24 GB of video memory, and 120 GB of RAM. For complex tasks, the agent took approximately 2 to 4 hours to converge, whereas for simpler tasks, convergence was achieved within 10 to 30 minutes. Table \ref{tab:exp_env} provides a detailed description of the experimental environment.

\begin{table}[h]
\caption{Experimental Environment}
\label{tab:exp_env}
\centering
\begin{tabular}{ll}
\toprule
Resource & Specification \\
\midrule
CPU & 48-core Intel Xeon E5-2666 v4 \\
GPU & NVIDIA GeForce RTX 4090 (24 GB) \\
RAM & 118.1 GB \\
Convergence Time (Complex Tasks) & 2-4 hours \\
Convergence Time (Simple Tasks) & 10-30 minutes \\
\bottomrule
\end{tabular}
\end{table}

\paragraph{Hyperparameters} In our algorithm design, we did not introduce a coefficient for the LLM heuristic term, eliminating the need for parameter adjustments for each task. Generally, the stronger the language model's ability to understand the environment, the greater the performance improvement. LLM-TD3 is built on top of TD3, and doesn't require parameter tuning.  In the baseline implementation, TD3's hyperparameters are also fixed for comparison. The hyperparameters of LLM-TD3 are detailed in Table \ref{tab:hyperparameters}.

\begin{table}[h]
\caption{Hyperparameters of LLM-TD3}
\label{tab:hyperparameters}
\centering
\begin{tabular}{ll}
\toprule
Hyperparameter & Value \\
\midrule
LLM Type & ChatGPT-Classic \\
Start Timesteps & 25,000 \\
Evaluation Frequency & 5,000 \\
Exploration Noise (Std) & 0.1 \\
Batch Size & 256 \\
Discount Factor $\gamma$ & 0.99 \\
Target Network Update Rate (Tau) & 0.005 \\
Policy Noise & 0.2 \\
Noise Clip & 0.5 \\
Policy Update Frequency & 2 \\
Hidden Layer Size & 512 (10,240 for Humanoid)  \\
\bottomrule
\end{tabular}
\end{table}

Since SAC only provides data for five Mujoco tasks, we used this data as the baseline and implemented SAC for the remaining tasks based on the description provided in \cite{sac}. The hyperparameters used for this implementation are detailed in Table \ref{tab:sac_hyper}. For the implementation of PPO in continuous action space, please refer to \href{https://docs.cleanrl.dev/rl-algorithms/ppo/}{https://docs.cleanrl.dev/rl-algorithms/ppo/}.

\begin{figure}[h] 
    \begin{minipage}[t]{0.48\linewidth}
        \centering
        \captionof{table}{Hyperparameters of SAC}
        \label{tab:sac_hyper}
        \begin{tabular}{ll}
        \toprule
        Hyperparameter & Value \\
        \midrule
        Critic Learning Rate & 3e-3 \\
        Actor Learning Rate & 3e-4 \\
        Entropy Target & $- \dim(\mathcal{A})$\\
        Policy Update Frequency & 1 \\
        Reward Scale & $\frac{1}{8}$, 1 \\
        Hidden Layer Size & 128  \\
        \bottomrule
        \end{tabular}
    \end{minipage}
    \hfill
    \begin{minipage}[t]{0.48\linewidth}
        \centering
        \captionof{table}{Convergence Line for Each Task}
        \label{tab:convergence}
        \begin{tabular}{ll}
        \hline
        Task & Convergence Line \\
        \hline
        Humanoid & 4000 \\
        Ant & 4400 \\
        Half-Cheetah & 9600  \\
        MountainCarContinuous & 1 \\
        Walker2D & 3200 \\
        Hopper & 2400 \\
        Pendulum & -300 \\
        InvertedPendulum & 800 \\
        \hline
        \end{tabular}
    \end{minipage}
\end{figure}

\paragraph{Performance and Deviation}
We present the performance of LLM-TD3 and baseline algorithms, including their standard deviations, in Figure \ref{fig:performances}. As mentioned in Section \ref{sec:exp}, an agent is considered to have converged when it reaches 80\% of peak performance. The convergence lines for each task are listed in Table \ref{tab:convergence}.

LLM-TD3 demonstrates state-of-the-art performance on 7 out of 8 tasks. Unlike baseline algorithms, which tend to perform poorly on some tasks, LLM-TD3 consistently achieves state-of-the-art results across all tasks without requiring hyperparameter adjustments. Its ability to maintain high performance uniformly across diverse tasks underscores LLM-TD3's robustness and generalizability. This robustness positions LLM-TD3 as a promising candidate for future applications in broader environments, such as web agents and autonomous driving systems.


\begin{figure}[h!]
    \centering

    \begin{subfigure}[b]{0.24\textwidth}
        \centering
        \includegraphics[width=\textwidth]{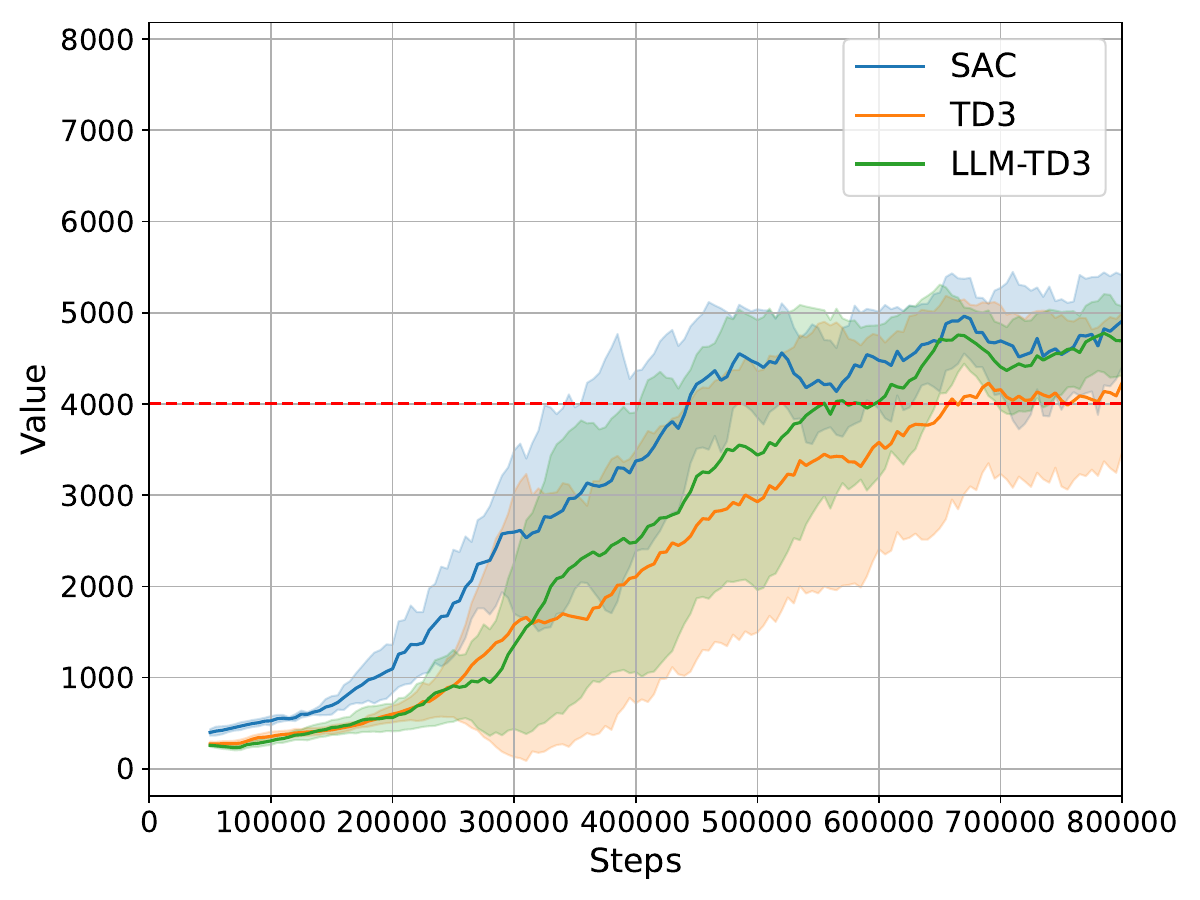}
        \caption{Humanoid}
        \label{fig:fig1}
    \end{subfigure}
    \hfill
    \begin{subfigure}[b]{0.24\textwidth}
        \centering
        \includegraphics[width=\textwidth]{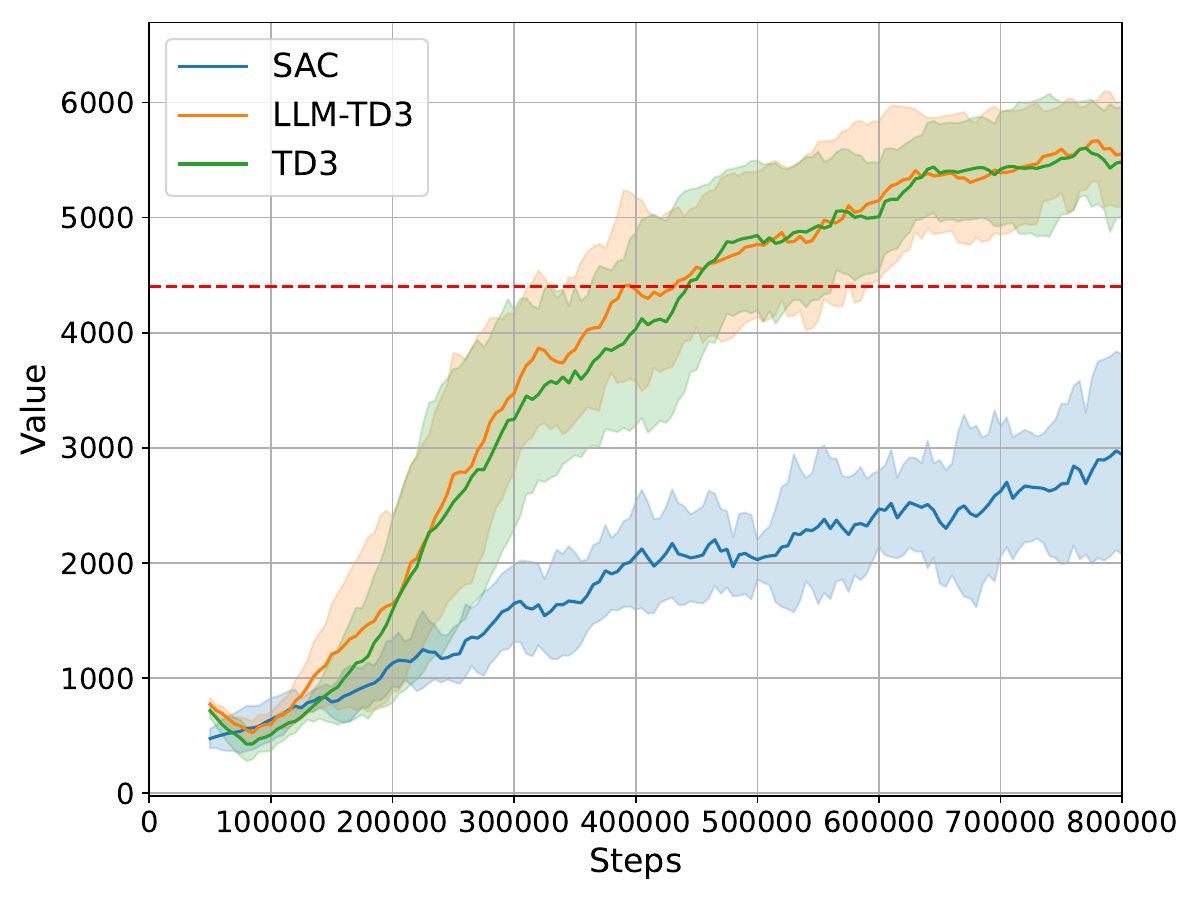}
        \caption{Ant}
        \label{fig:fig2}
    \end{subfigure}
    \hfill
    \begin{subfigure}[b]{0.24\textwidth}
        \centering
        \includegraphics[width=\textwidth]{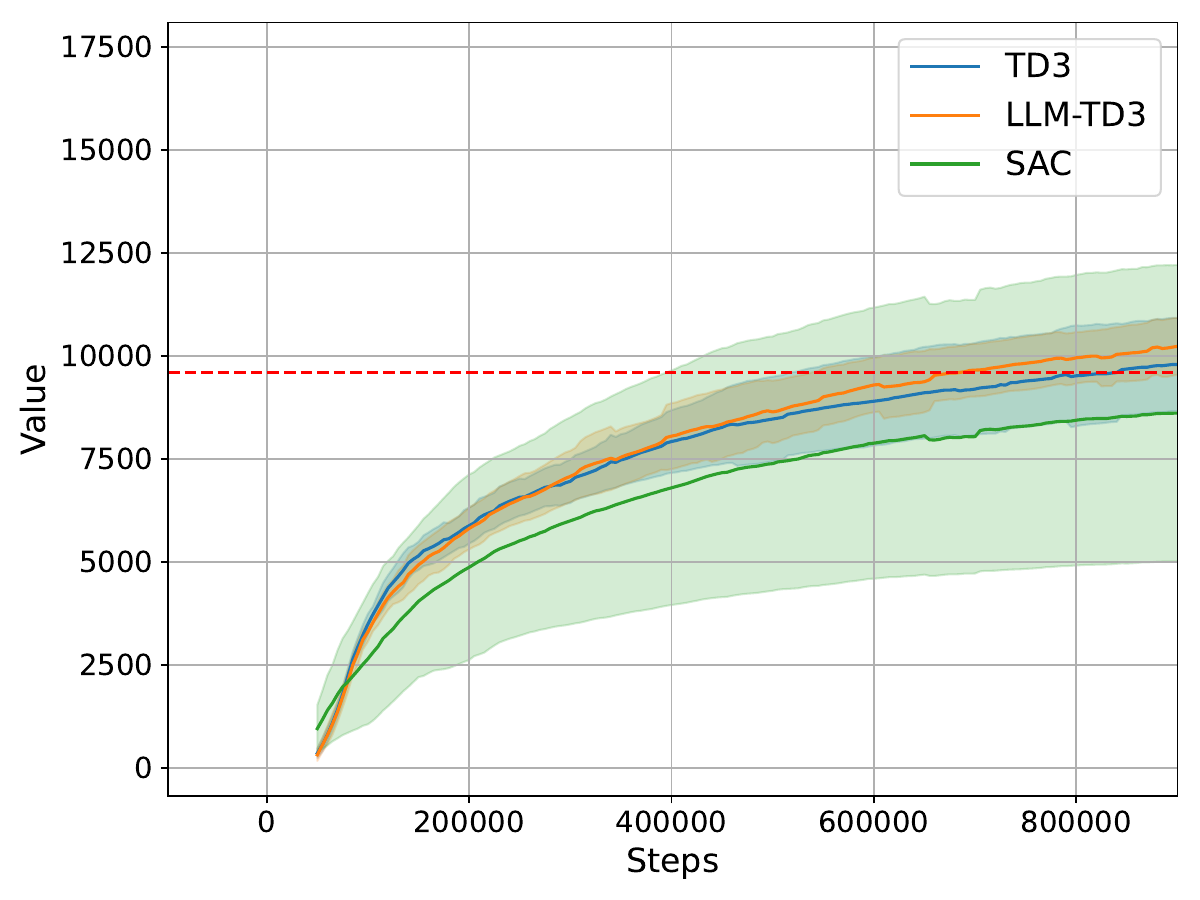}
        \caption{Half-Cheetah}
        \label{fig:fig3}
    \end{subfigure}
    \hfill
    \begin{subfigure}[b]{0.24\textwidth}
        \centering
        \includegraphics[width=\textwidth]{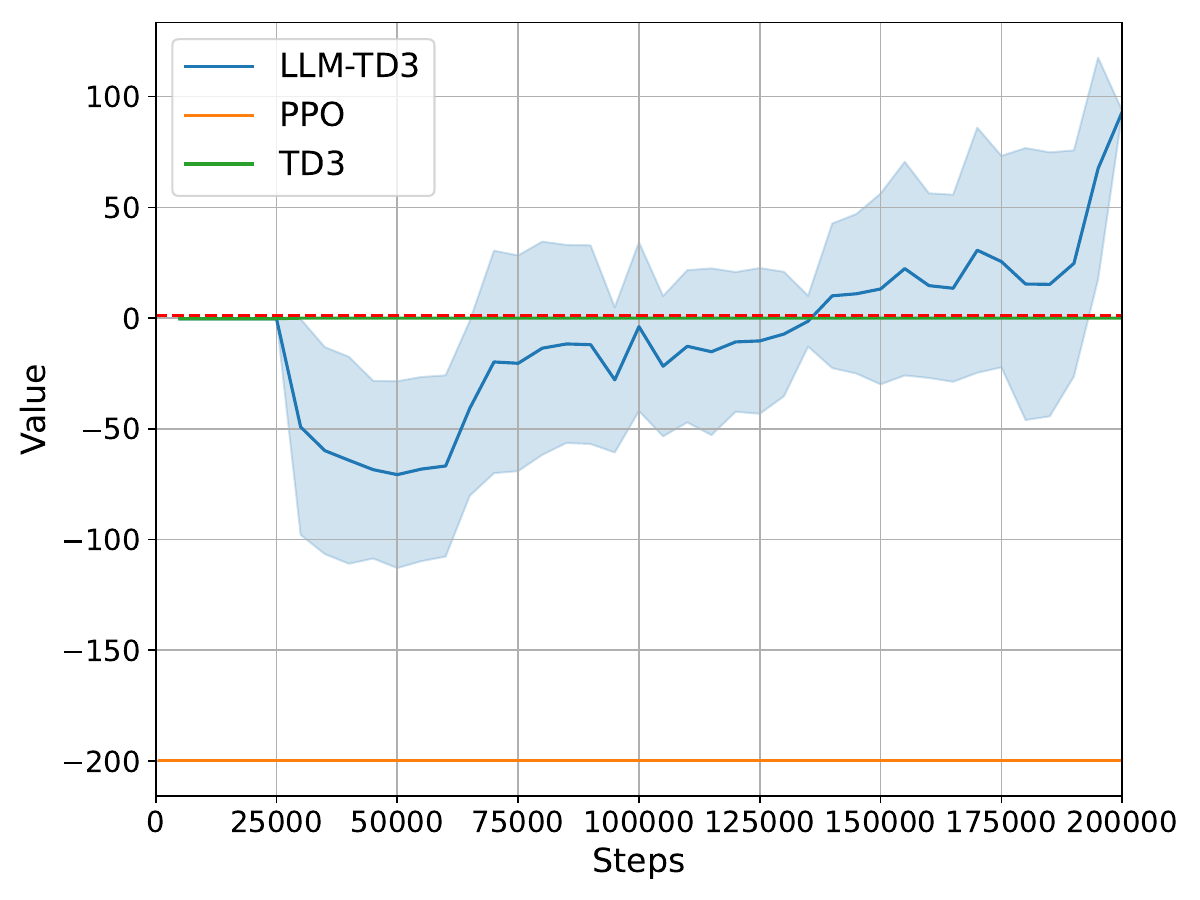}
        \caption{MountainCar}
        \label{fig:fig8}
    \end{subfigure}

    \vskip\baselineskip 

    \begin{subfigure}[b]{0.24\textwidth}
        \centering
        \includegraphics[width=\textwidth]{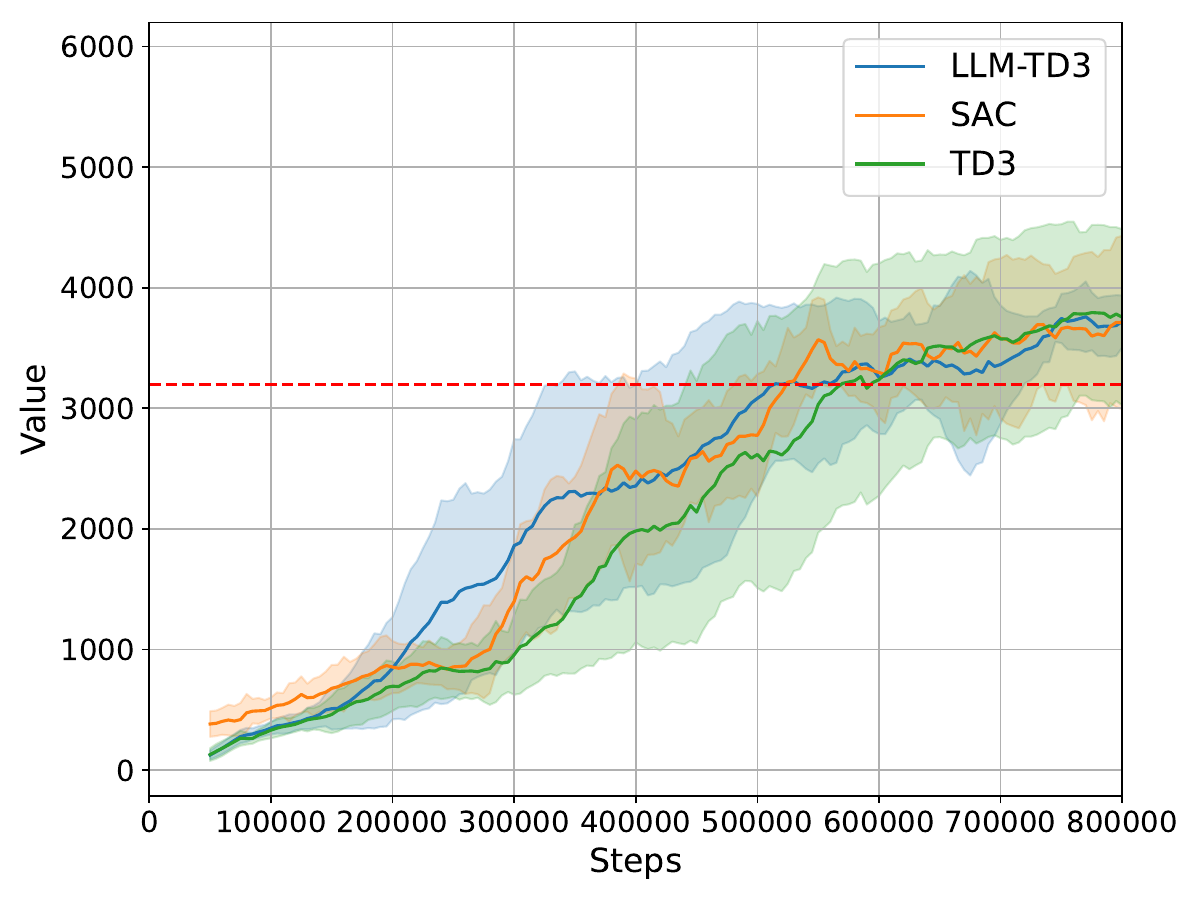}
        \caption{Walker2D}
        \label{fig:fig4}
    \end{subfigure}
    \hfill
    \begin{subfigure}[b]{0.24\textwidth}
        \centering
        \includegraphics[width=\textwidth]{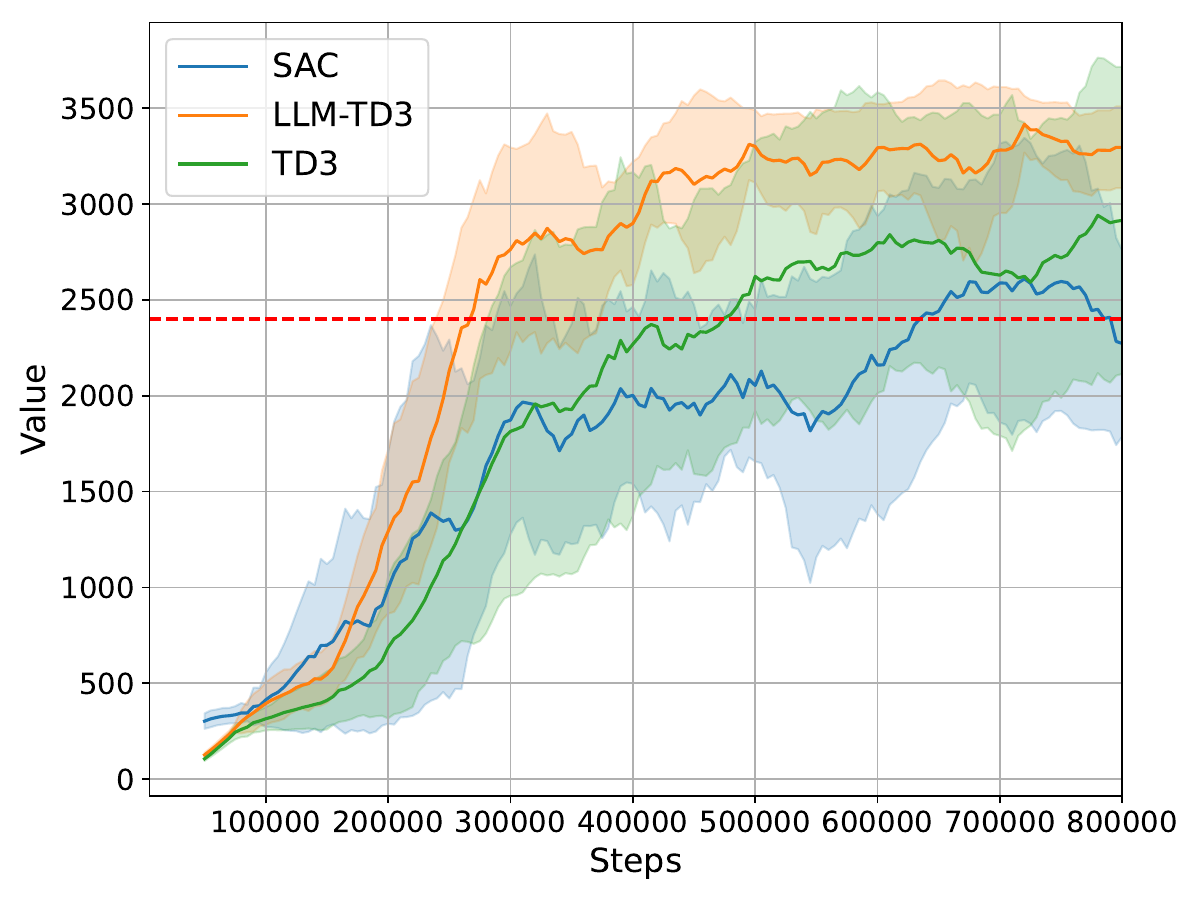}
        \caption{Hopper}
        \label{fig:fig5}
    \end{subfigure}
    \hfill
    \begin{subfigure}[b]{0.24\textwidth}
        \centering
        \includegraphics[width=\textwidth]{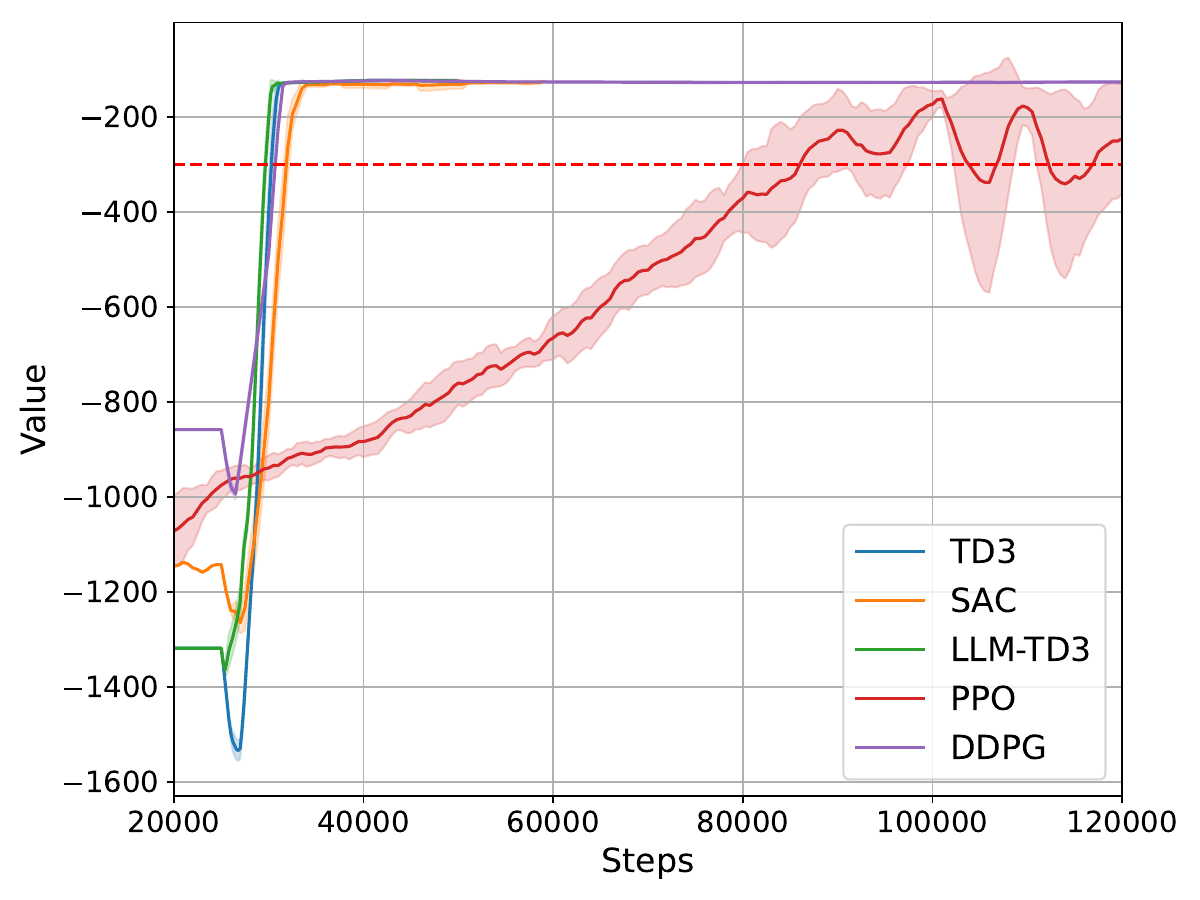}
        \caption{Pendulum}
        \label{fig:fig6}
    \end{subfigure}
    \hfill
    \begin{subfigure}[b]{0.24\textwidth}
        \centering
        \includegraphics[width=\textwidth]{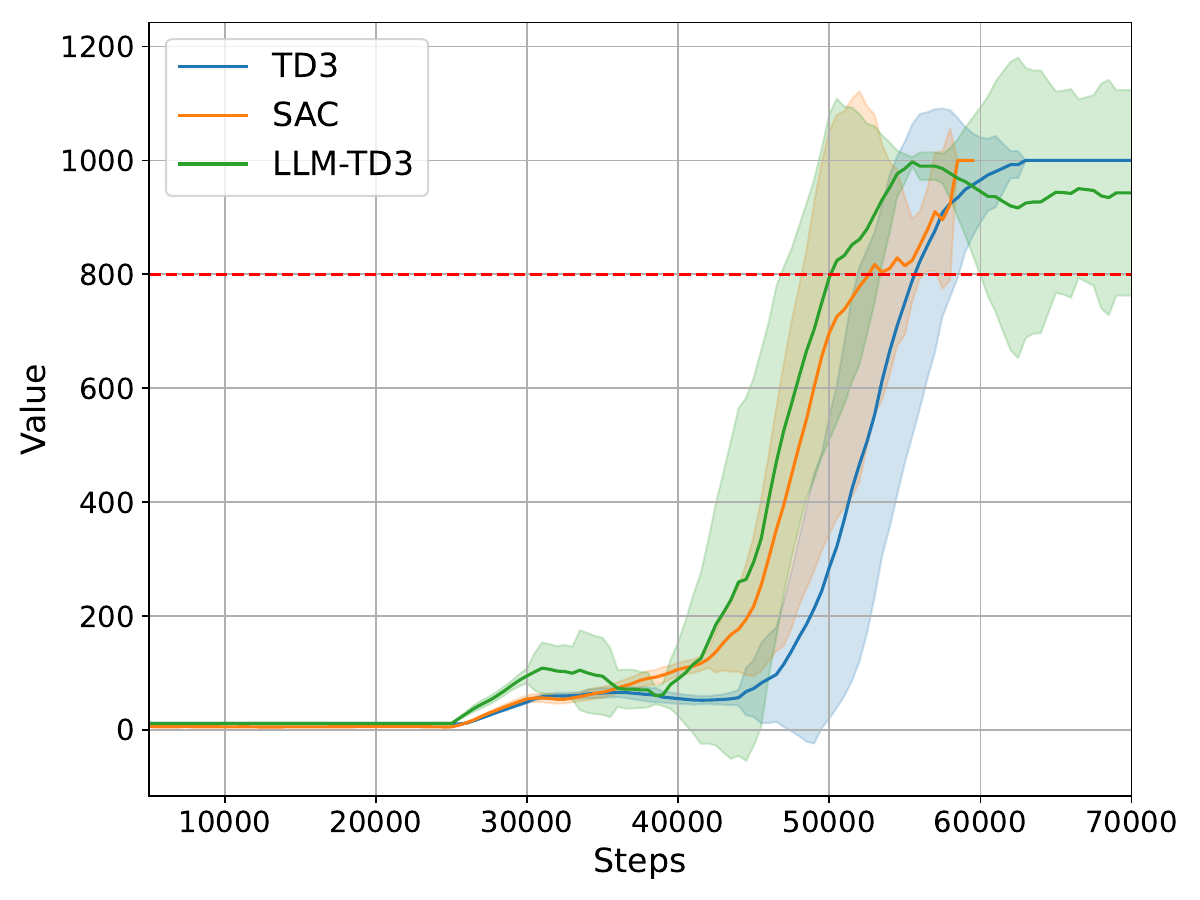}
        \caption{InvertedPendulum}
        \label{fig:fig7}
    \end{subfigure}

    \caption{Experimental results of different control tasks. Our proposed LLM-TD3 is able to converge quickly on different types of control tasks.}
    \label{fig:performances}
\end{figure}

\section{Discussion and Limitation}
This work integrates the advantages of generative models and reinforcement learning, addressing the limitations of both approaches. Its primary contribution is enabling agents to learn from imprecise data. 
By utilizing TD error-based Q-value iteration, the method transforms the cost induced by hallucination into a cost of exploration.
Specifically, it shifts the burden of creating precise, high-quality, target-aligned datasets to the cost of exploration, thus facilitating development in areas like imitation learning and generative model enhanced reinforcement learning. 
Moreover, when extensive exploration is allowed, agents can discover superior trajectories based on guidance trajectories, thereby achieving or even surpassing expert performance in many tasks more quickly.
Due to the limitations of current large language models (LLMs), we have not tested the framework in visual environments. Additionally, the quality of heuristic values provided by the LLM decreases as task complexity increases, and LLM-TD3 fails to surpass SAC in the humanoid task. One possible solution is to adopt more powerful language models or provide high-quality feedback to the LLM to enhance its guidance.

Theoretically, as long as the LLM can understand the environment, our framework can improve sample efficiency. We have only tested its capability on classic control tasks; future directions may include tasks such as web control and autonomous driving.

Furthermore, our training framework supports the development of autonomous agents, such as autonomous soccer robots \cite{abreu2023designing,deepmind_soccer_sr}. A sparse reward function can be implemented for these robots, offering rewards for goals and penalties for fouls. The provided guidance is then translated into heuristic Q-values through generative models to enhance Q-learning. Ultimately, the robots can identify the most effective goal-scoring strategies through limited interactions.

\end{document}